\documentclass[sigconf]{acmart}
\usepackage[utf8]{inputenc}
\usepackage{array} 
\usepackage{multirow} 
\usepackage{rotating} 
\usepackage{enumitem}
\usepackage{placeins}  
\usepackage{xcolor}
\usepackage{fontawesome5}
\usepackage[table,xcdraw]{xcolor}
\usepackage[font=small,skip=2pt]{caption}
\usepackage{amsmath}
\usepackage{adjustbox}
\usepackage{footnote}
\usepackage{graphicx}
\usepackage{algorithm}
\usepackage{algorithmic}
\newtheorem{lemma}{Lemma}
\newtheorem{theorem}{Theorem}
\newtheorem{assumption}{Assumption}
\newtheorem{corollary}{Corollary}

\usepackage{subcaption}
\usepackage{float}  
\usepackage{nicematrix}
\usepackage{booktabs,multirow}
\usepackage{tabularray}
\usepackage{collcell}
\usepackage{makecell}
\usepackage{bm}
\usepackage{tabularx}
\AtBeginDocument{%
  }

\acmConference[CIKM '26]
  {The 35th ACM International Conference on Information and Knowledge Management}
  {November 7--11, 2026}
  {Rome, Italy}

\setcopyright{none}
\makeatletter
\let\originalcopyrightpermissionfootnote\footnotetextcopyrightpermission
\renewcommand{\footnotetextcopyrightpermission}[1]{%
  \originalcopyrightpermissionfootnote{%
    \textit{Accepted at the 35th ACM International Conference on Information
    and Knowledge Management (CIKM '26), Rome, Italy.}}%
}
\makeatother

\begin{document}

\title[Retrieval--Corrected Conformal Prediction for Time Series]{Retrieval--Corrected Conformal Prediction for Time Series}

\author{Sangjin Jin}
\email{sj.jin@unist.ac.kr}
\affiliation{
  \institution{Ulsan National Institute of Science and Technology}
  \city{Ulsan}
  \country{Republic of Korea}
}

\author{Kangmin Kim}
\email{kmin1015@unist.ac.kr}
\affiliation{
  \institution{Ulsan National Institute of Science and Technology}
  \city{Ulsan}
  \country{Republic of Korea}
}

\author{Junhyeong Lee}
\email{jun.lee@unist.ac.kr}
\affiliation{
  \institution{Ulsan National Institute of Science and Technology}
  \city{Ulsan}
  \country{Republic of Korea}
}

\author{Yongjae Lee}
\authornote{Corresponding author.}
\email{yongjaelee@unist.ac.kr}
\affiliation{
  \institution{Ulsan National Institute of Science and Technology}
  \city{Ulsan}
  \country{Republic of Korea}
}
\affiliation{
  \institution{LinqAlpha}
  \city{New York}
  \state{NY}
  \country{USA}
}

\vspace{-5pt}

\setlength{\textfloatsep}{3pt plus 0pt minus 1pt}
\setlength{\intextsep}{4pt plus 0pt minus 1pt}

\begin{abstract}

Conformal prediction (CP) provides distribution-free prediction intervals for fixed forecasters, but its standard calibration procedure is often inefficient for time series data, where forecast errors are temporally dependent and change across time and operating conditions. Recent time series CP methods improve local calibration using recent, weighted, or localized residuals. Yet local calibration can remain indirect, since broad residual weighting or additional adaptation procedures may dilute the evidence most relevant to the current prediction. This motivates a simple retrieval and correction strategy that selects similar past residuals as local evidence and then corrects the coverage error left by retrieval. In this paper, we propose \emph{Retrieval--Corrected Conformal Prediction} (RCCP), a retrieval-augmented calibration method for time series prediction intervals. RCCP builds an asymmetric interval from retrieved one-sided residuals and calibrates its normalized retrieval error with a scalar conformal correction. Thus, retrieval provides local residual evidence, while conformal correction determines the final scale needed for coverage. We provide a coverage-gap bound based on the stability of the normalized retrieval error distribution. Across standard benchmarks and backbone forecasters, RCCP attains the target coverage in every setting and achieves the lowest Winkler scores, with fewer severe misses. RCCP also achieves low calibration and inference overhead, showing that retrieval-corrected calibration is an effective and scalable approach to uncertainty quantification in time series forecasting. Code is available at \href{https://github.com/jinsaaang/rccp}{https://github.com/jinsaaang/rccp}.

\end{abstract}



\maketitle

\section{Introduction}

Time series forecasting plays a central role in decision making across domains such as finance, traffic, and supply chain systems~\cite{xiang2022temporal,tang2022domain,chopra2019supply}. Modern deep forecasting models have improved point prediction accuracy across these settings~\cite{lim2021time,benidis2022deep,kim2025comprehensive}, yet point accuracy alone is insufficient for decision making under uncertainty. Forecast users also need prediction intervals that quantify plausible deviations from the point forecast and attain the desired empirical coverage. An informative interval should be narrow when the local error scale is small, wider when current conditions imply larger errors, and calibrated to the target level. Probabilistic forecasting approaches estimate predictive distributions or conditional quantiles~\cite{salinas2020deepar,lim2021temporal}, but their uncertainty estimates depend on the model class, training objective, and distributional assumptions. When these assumptions are misspecified or the test distribution differs from training, nominal uncertainty may not translate into reliable empirical coverage. This motivates predictive inference based on calibration for fixed forecasters.

Conformal prediction (CP) provides a distribution free framework for constructing prediction intervals with a prescribed coverage level~\cite{vovk2005algorithmic,lei2018distribution,angelopoulos2023conformal}. For a fixed regression forecaster, split CP computes nonconformity scores on a calibration set, often using absolute prediction residuals, and takes an empirical quantile of these scores as the interval threshold. Under exchangeability of calibration and test scores, this procedure attains finite sample marginal coverage. Its appeal is that coverage is obtained from calibration scores rather than a correctly specified predictive distribution. In time series, however, the residual process is generally nonexchangeable. Forecast errors are temporally dependent, heteroskedastic, and affected by nonstationarity, distribution shifts, and horizon dependent uncertainty~\cite{enbpi,aci,zaffran2022adaptive,lin2022conformal,spci}. A global threshold computed from all calibration residuals therefore mixes errors from different temporal conditions. It can be conservative in stable periods and insufficient in volatile periods, which weakens the coverage and efficiency tradeoff required of prediction intervals.

Recent time series CP methods address this mismatch by adapting calibration to the current prediction context~\cite{nexcp}. They either update calibration through recent errors or miscoverage feedback~\cite{aci,zaffran2022adaptive,angelopoulos2023pid,gibbs2024conformal}, or construct weighted and localized empirical distributions over conformity scores using temporal proximity, learned similarity, semantic features, or state representations~\cite{tibshirani2019conformal,hopcpt,ctssf,kowcpi,corel,rescp}. These methods avoid calibrating against the average residual behavior of the entire calibration period. This direction is essential for sharp intervals, since calibration should rely on residual information that reflects the current error distribution. A remaining limitation is that local thresholds are often obtained by softly weighting many calibration samples or by introducing an additional adaptation procedure. When only a smaller set of past errors is relevant to the current prediction, weakly related residuals can still affect the interval and reduce sharpness. Retrieval provides a more direct way to use residual evidence from similar past prediction contexts, but the resulting interval is not calibrated by itself. The key problem is therefore how to use retrieved residuals for a sharp local interval while correcting the coverage error that remains after retrieval.

Motivated by this problem, we propose \emph{Retrieval--Corrected Conformal Prediction} (RCCP). The main idea is to separate residual retrieval from conformal correction. Retrieval augmented time series forecasting has shown that similar historical windows or reference trajectories can provide useful nonparametric information for prediction~\cite{retime,ratd,raft,timerag,tsrag}. RCCP uses this principle for uncertainty estimation rather than point forecasting. Given a fixed forecaster, RCCP retrieves similar past prediction contexts and their realized residuals from a time ordered knowledge base, then builds an asymmetric local interval from positive and negative residuals. This interval captures local residual behavior without training an additional weighting model or relying on a global residual quantile.

The retrieved interval is constructed from data dependent local evidence, so RCCP does not treat it as the final conformal interval. It instead evaluates the remaining retrieval error on calibration data by comparing the realized residual with the retrieved upper or lower scale. The resulting normalized retrieval error is calibrated into a scalar correction factor for test time intervals. As a result, retrieval determines the local asymmetric shape of the interval, while conformal correction determines the final scale needed for coverage. Section~\ref{sec:retrieval_correction} formalizes this separation by first constructing retrieved residual quantiles and then correcting them through the normalized retrieval error.

\vspace{3pt}
Our main contributions are summarized as follows.
\vspace{-6pt}
\begin{itemize}
    \item We apply residual retrieval to conformal prediction as a direct alternative to broad weighting over calibration samples and additional training procedures, selecting residuals from similar past prediction contexts as local evidence.
    \item We propose RCCP, a framework that combines asymmetric retrieved intervals with scalar conformal correction.
    \item Theoretical results establish a coverage gap bound and an asymptotic coverage result for the RCCP interval, based on stability of the normalized retrieval error distribution.
    \item Empirical results show that RCCP improves the coverage and interval efficiency tradeoff with fewer large misses while keeping calibration cost low.
\end{itemize}



%


\section{Related Work}

\paragraph{\textbf{Conformal prediction for time series forecasting.}}
Conformal prediction provides a general framework for constructing prediction sets with coverage guarantees under exchangeability, and has been widely studied as a model-agnostic approach to uncertainty quantification \citep{vovk2005algorithmic, shafer2008tutorial, lei2018distribution, angelopoulos2023conformal}. 
However, time series forecasting violates exchangeability assumption due to temporal dependency, heteroskedasticity, and distribution shifts. 
Prior sequential conformal prediction for time series forecasting address this issue through residual-set updates~\citep{enbpi}, adaptive miscoverage control~\citep{aci}, or future residual-quantile modeling~\citep{spci}, including EnbPI, ACI, and SPCI. 
Another branch of time series conformal methods localizes calibration around
the current test point. NexCP~\citep{nexcp} follows this view by constructing a
test-dependent weighted residual distribution, assigning larger influence to
temporally closer available errors when computing the conformal quantile.
Recent local methods further improve efficiency by selecting calibration residuals according to their relevance to the current forecasting context, rather than relying only on temporal proximity. HopCPT~\citep{hopcpt} uses Modern Hopfield Networks for similarity-based sample reweighting, ResCP~\citep{rescp} uses reservoir-state similarity to reweight residuals, KOWCPI~\cite{kowcpi} adapts Reweighted Nadaraya--Watson estimation for weighted quantile regression on dependent scores, and CT-SSF~\cite{ctssf} learns dynamic weights in a semantic feature space. These methods demonstrate the value of local calibration, but they largely focus on estimating a test-specific conformal quantile. Consequently, their coverage behavior may become sensitive to the quality and effective support of the retrieved or weighted calibration examples.

\paragraph{\textbf{Retrieval-augmented time series forecasting.}}
In recent years, retrieval-augmented generation (RAG) has been widely used in language
models to complement parametric knowledge with relevant external context
retrieved at inference time~\cite{lewis2020retrieval, fan2024survey, ram2023context}. Motivated by this paradigm, retrieval
augmentation has been introduced to time series forecasting as an external
memory mechanism. ReTime~\cite{retime} shows that retrieving relevant reference time
series can reduce forecasting uncertainty, especially for long-horizon
prediction. RAFT~\cite{raft} further retrieves historical candidates with input patterns
similar to the query and uses their future trajectories as predictive
inductive bias. Beyond deterministic forecasting, RATD~\cite{ratd} uses retrieved
references to guide the denoising process in diffusion-based forecasting.
TS-RAG~\cite{tsrag} extends the same principle to time series foundation models by
retrieving semantically relevant segments from a knowledge base and fusing
them with model representations for zero-shot forecasting. These studies
suggest that retrieval can provide external historical patterns that are
difficult to fully capture through parametric training alone.

\begin{figure*}[t]
    \centering
    \includegraphics[width=0.97\textwidth]{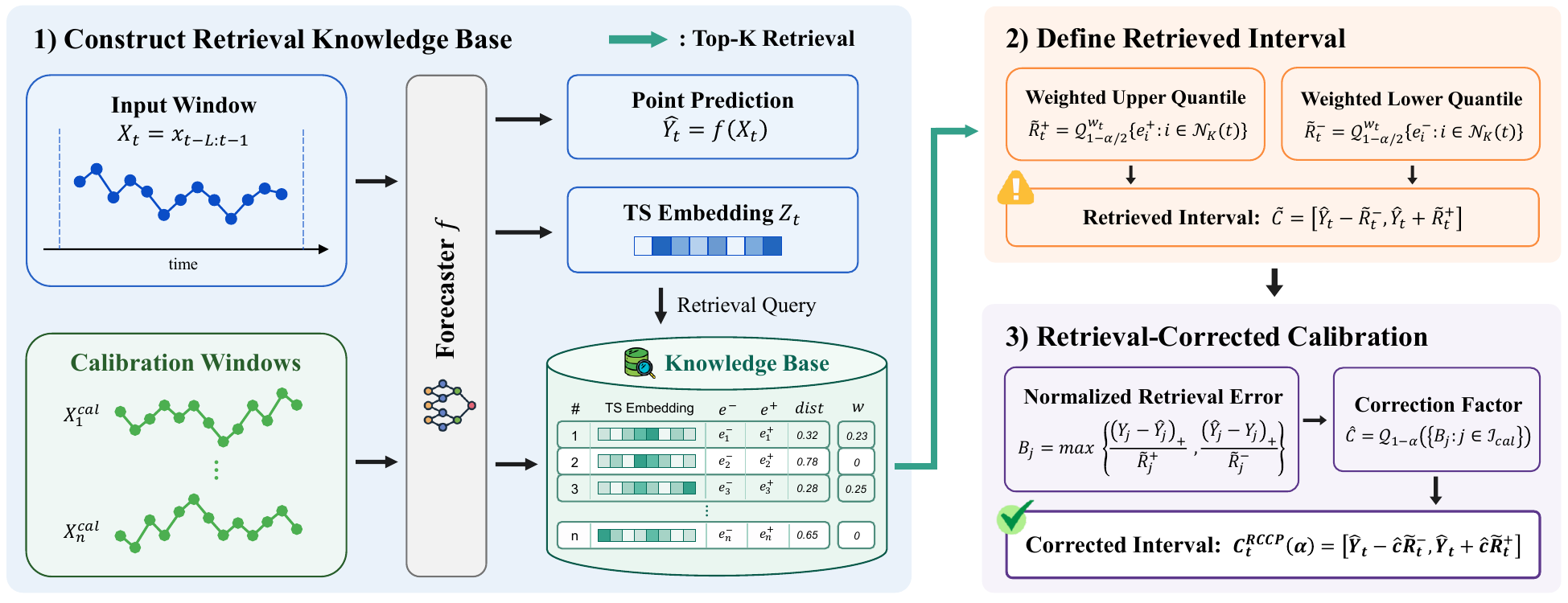}
    \caption{
    Overview of Retrieval--Corrected Conformal Prediction. RCCP builds a retrieval knowledge base from evaluated prediction contexts, forms an asymmetric retrieved interval from one-sided residuals of similar contexts, and calibrates the interval with a scalar correction estimated from normalized retrieval errors.
    }
    \label{fig:rccp_overview}
\end{figure*}

\section{Retrieval--Corrected Conformal Prediction}

In this section, we present \emph{Retrieval-Corrected Conformal Prediction} (RCCP). RCCP separates local residual evidence from the conformal correction needed for coverage. We first introduce the problem setting and then formalize the retrieval and correction steps.

\subsection{Problem Definition}

We consider a time series regression setting with observations $\{(X_t,Y_t)\}_{t=1}^T$ indexed by time. At prediction time $t$, the input $X_t\in\mathcal X$ contains information available before the response $Y_t\in\mathbb R$ is revealed, such as a lookback window, known exogenous variables, calendar features, or model states computable at prediction time. A base forecaster $f:\mathcal X\to\mathbb R$ is trained on a proper training period and then fixed during calibration and testing, producing the point prediction $\widehat Y_t=f(X_t)$. For a miscoverage level $\alpha\in(0,1)$, the goal is to construct a prediction interval $C_t^\alpha(X_t)$ satisfying the marginal coverage criterion
\begin{equation}
    \mathbb P\left(Y_t\in C_t^\alpha(X_t)\right)\ge 1-\alpha.
    \label{eq:coverage_objective}
\end{equation}

After the forecaster is fixed, calibration uses only held-out residuals. We assess validity by coverage and efficiency by interval length.

\paragraph{Split conformal prediction.}
The observations are divided in time order into a proper training period $\mathcal D_{\mathrm{train}}$, a calibration period $\mathcal D_{\mathrm{cal}}$, and a test period $\mathcal D_{\mathrm{test}}$, with calibration and test indices $\mathcal I_{\mathrm{cal}}$ and $\mathcal I_{\mathrm{test}}$. For regression, the standard split conformal score is the absolute residual $S_t=|Y_t-\widehat Y_t|$ for $t\in\mathcal I_{\mathrm{cal}}$. Split conformal prediction estimates the forecast-error scale by the empirical quantile $\widehat q_{\alpha}=Q_{1-\alpha}(\{S_t:t\in\mathcal I_{\mathrm{cal}}\})$, where $Q_{1-\alpha}$ denotes the conformal empirical quantile under the standard finite-sample convention. The corresponding interval is
\begin{equation}
    C_t^{\mathrm{split}}(\alpha)
    =[\widehat Y_t-\widehat q_{\alpha},\widehat Y_t+\widehat q_{\alpha}].
    \label{eq:split_interval}
\end{equation}

Under exchangeability of the calibration and test scores, Eq.~\eqref{eq:split_interval} satisfies the coverage objective in Eq.~\eqref{eq:coverage_objective}.

For time series forecasting, the residual sequence is temporally dependent and the error distribution can change over time. Consequently, a single global residual quantile may be poorly matched to the uncertainty at a particular test time. Weighted and localized conformal methods address this by assigning nonuniform influence to calibration residuals or by conditioning residual quantiles on recent observations, temporal states, or learned features~\cite{nexcp,kowcpi,rescp,ctssf}. RCCP follows the same motivation for local adaptivity, but separates the construction of local residual evidence from the final coverage calibration. It first constructs a local interval from retrieved residuals and then calibrates the remaining error of that interval.

\subsection{Retrieval-Augmented Local Interval}
\label{sec:interval_retreival}
The retrieval stage provides the local residual evidence used by RCCP. The current prediction context is encoded as a retrieval key, matched against a time-ordered knowledge base of evaluated residuals, and summarized through asymmetric residual quantiles. The goal of this stage is local adaptivity, not coverage certification.

\paragraph{Retrieval knowledge base construction.}
To construct the retrieval knowledge base, RCCP encodes each prediction context into a retrieval key $z_t$ before $Y_t$ is observed. The key can be derived from an input-window embedding, a forecaster representation, point prediction features, exogenous variables, or their concatenation. We write $z_t=\psi_f(X_t,\widehat Y_t)$, where $\psi_f$ may use any predictable forecaster features. Once the response at time $i$ is observed, RCCP appends the key together with the realized one-sided residuals, $e_i^+=(Y_i-\widehat Y_i)_+$ and $e_i^-=(\widehat Y_i-Y_i)_+$. The two residuals are stored separately to allow different upper and lower tail behavior. Before time $t$, the knowledge base is $\mathcal K_{t-1}=\{(z_i,e_i^+,e_i^-):i<t\}$. This ordering ensures that the interval at time $t$ uses only past residuals.

\paragraph{Residual retrieval.}
RCCP forms a local residual set by comparing the query $z_t$ with the keys stored in $\mathcal K_{t-1}$. Let $d(z_t,z_i)$ denote a retrieval distance between the current context and a past context, where smaller values indicate closer contexts. In our experiments, we instantiate $d$ as Euclidean distance. We denote the retrieval operation by $\mathcal N_K(t)=\operatorname{KNN}_K(z_t,\mathcal K_{t-1},d)$, which returns the $K$ past indices with the smallest distances to $z_t$. This step localizes calibration to residuals associated with similar past prediction contexts. State-based time series CP uses an analogous principle by reweighting residuals according to proximity in a learned or constructed state space~\cite{rescp}. In RCCP, this neighborhood determines only the local interval, and the coverage correction is computed separately.

The retrieved residuals may be treated uniformly or assigned distance-based weights. In the weighted case,
\begin{equation}
    w_i(t)=
    \frac{\exp(-d(z_t,z_i)/\tau)}
    {\sum_{j\in\mathcal N_K(t)}\exp(-d(z_t,z_j)/\tau)},
    \qquad i\in\mathcal N_K(t),
    \label{eq:weights}
\end{equation}

where $\tau>0$ controls the concentration of the local empirical distribution. Uniform weights recover the hard top-$K$ case. In both cases, $w_t$ specifies the empirical distribution used to compute local residual quantiles.

\paragraph{Asymmetric retrieved interval.}
RCCP summarizes the retrieved residuals separately on the upper and lower sides. This avoids imposing symmetry on the forecast error and allows the interval to adapt differently to positive and negative deviations. The retrieved quantiles are
\begin{equation}
\begin{aligned}
    \widetilde R_t^+
    &=Q_{1-\alpha/2}^{w_t}\{e_i^+:i\in\mathcal N_K(t)\},\\
    \widetilde R_t^-
    &=Q_{1-\alpha/2}^{w_t}\{e_i^-:i\in\mathcal N_K(t)\},
\end{aligned}
\label{eq:retrieved_quantiles}
\end{equation}
and they define the retrieved interval
\begin{equation}
    \widetilde C_t
    =[\widehat Y_t-\widetilde R_t^-,\widehat Y_t+\widetilde R_t^+].
    \label{eq:retrieved_interval}
\end{equation}

This construction is related to adaptive lower and upper interval estimates in conformalized quantile regression, but the interval here is obtained from retrieved residuals rather than learned conditional quantile functions~\cite{cqr}.

\paragraph{Need for conformal correction.}
The retrieved interval can be locally adaptive because it uses residuals from contexts close to the current prediction. In conformal prediction, however, locality alone does not imply calibration. The retrieved set is finite and data-dependent, and its quantiles can be unstable when the effective support is small or biased when proximity in the key space does not match future error behavior. RCCP therefore evaluates Eq.~\eqref{eq:retrieved_interval} on calibration data and corrects its remaining error, rather than using it as the final conformal interval.

\subsection{Retrieval-Corrected Calibration}
\label{sec:retrieval_correction}

After retrieval has produced a local interval, the correction step calibrates its remaining error. Rather than calibrating raw residual magnitudes, RCCP calibrates residuals after normalization by the interval produced by retrieval. This makes the correction depend on how well retrieval estimated the local scale, instead of discarding the local information and returning to a global residual threshold.

\paragraph{Normalized retrieval error and correction.}
For each calibration index $j\in\mathcal I_{\mathrm{cal}}$, the retrieved interval $\widetilde C_j$ is computed before observing $Y_j$. Once $Y_j$ is revealed, RCCP computes
\begin{equation}
    B_j=
    \max\left\{
    \frac{(Y_j-\widehat Y_j)_+}{\widetilde R_j^+},
    \frac{(\widehat Y_j-Y_j)_+}{\widetilde R_j^-}
    \right\}.
    \label{eq:rccp_score}
\end{equation}
The first ratio measures the positive residual relative to the retrieved upper scale, while the second ratio measures the negative residual relative to the retrieved lower scale. Thus, $B_j$ is a scalar nonconformity score for the asymmetric interval. RCCP then estimates a scalar correction from the empirical distribution of these scores. Let $q=1-\alpha$, and let $q_n$ denote the finite-sample conformal quantile level. The correction is $\widehat c=\widehat F_{\mathrm{cal}}^{-1}(q_n)=Q_{1-\alpha}(\{B_j:j\in\mathcal I_{\mathrm{cal}}\})$, and the final interval is
\begin{equation}
    C_t^{\mathrm{RCCP}}(\widehat c)
    =[\widehat Y_t-\widehat c\,\widetilde R_t^-,
      \widehat Y_t+\widehat c\,\widetilde R_t^+].
    \label{eq:rccp_interval}
\end{equation}
Values of $\widehat c$ above one indicate that retrieved intervals tend to be too narrow on calibration data, while values below one indicate conservative retrieved intervals. Direct retrieval corresponds to fixing this multiplier at one, whereas RCCP estimates it from Eq.~\eqref{eq:rccp_score}.

\paragraph{Implementation.}
Algorithm~\ref{alg:rccp} summarizes the procedure. During calibration, RCCP constructs retrieved intervals, records normalized retrieval errors, and estimates $\widehat c$. During testing, the same retrieval rule constructs a retrieved interval for each new point and scales it by $\widehat c$. The knowledge base is updated only after the target is observed. Retrieval-based scores are computed only after the knowledge base contains at least $K$ evaluated residuals.

\begin{algorithm}[t]
\caption{Retrieval-Corrected Conformal Prediction}
\label{alg:rccp}
\begin{algorithmic}[1]
\STATE \textbf{Input:} forecaster $f$, key map $\psi_f$, distance $d$, knowledge base $\mathcal K$, indices $\mathcal I_{\mathrm{cal}},\mathcal I_{\mathrm{test}}$, level $\alpha$, neighbors $K$
\STATE \textbf{Output:} prediction intervals $C_t^{\mathrm{RCCP}}(\alpha)$
\STATE $\mathcal B \gets \emptyset$
\FOR{each $j \in \mathcal I_{\mathrm{cal}}$ in increasing order}
    \STATE $\widehat Y_j \gets f(X_j)$ and compute $z_j=\psi_f(X_j,\widehat Y_j)$
    \STATE $\mathcal N_K(j)\gets\operatorname{KNN}_K(z_j,\mathcal K,d)$
    \STATE compute $\widetilde R_j^+$ and $\widetilde R_j^-$ by Eq.~\eqref{eq:retrieved_quantiles}
    \STATE observe $Y_j$ and compute $e_j^+$, $e_j^-$, and $B_j$ by Eq.~\eqref{eq:rccp_score}
    \STATE $\mathcal B \gets \mathcal B\cup\{B_j\}$
    \STATE $\mathcal K \gets \mathcal K\cup\{(z_j,e_j^+,e_j^-)\}$
\ENDFOR
\STATE $\widehat c \gets Q_{1-\alpha}(\mathcal B)$
\FOR{each $t \in \mathcal I_{\mathrm{test}}$ in increasing order}
    \STATE $\widehat Y_t \gets f(X_t)$ and compute $z_t=\psi_f(X_t,\widehat Y_t)$
    \STATE $\mathcal N_K(t)\gets\operatorname{KNN}_K(z_t,\mathcal K,d)$
    \STATE compute $\widetilde R_t^+$ and $\widetilde R_t^-$ by Eq.~\eqref{eq:retrieved_quantiles}
    \STATE output $C_t^{\mathrm{RCCP}}(\alpha)$ by Eq.~\eqref{eq:rccp_interval}
    \STATE observe $Y_t$ and compute $e_t^+$ and $e_t^-$
    \STATE $\mathcal K \gets \mathcal K\cup\{(z_t,e_t^+,e_t^-)\}$
\ENDFOR
\end{algorithmic}
\end{algorithm}

\subsection{Theoretical Analysis}


We analyze the coverage behavior of RCCP without relying on exchangeability between calibration and test scores, which is unrealistic in time series settings~\cite{nexcp}. 
Instead, we establish a coverage-gap bound under stability conditions on the normalized retrieval error distribution, adapting standard non-exchangeable conformal prediction assumptions to our retrieval-normalized score setting~\cite{oliveira2024split,enbpi,sesia2020comparison}.
Theorem~\ref{thm:coverage_gap} decomposes the coverage error into calibration-test score mismatch and empirical correction error. Corollary~\ref{cor:asymptotic_coverage} shows that RCCP approaches the target coverage $1-\alpha$ when the normalized retrieval error is stable and the empirical correction is consistent.

\paragraph{Notation.}
Let $F_{\rm cal}(c)=\mathbb P_{\rm cal}(B\le c)$ and $F_{\rm test}(c)=\mathbb P_{\rm test}(B\le c)$ denote the calibration and test CDFs of the normalized retrieval error. Set $q=1-\alpha$, and let $c^*$ be the target calibration multiplier satisfying $F_{\rm cal}(c^*)=q$. We define the empirical calibration error and the finite-sample quantile-level error as
\begin{equation}
    e_n=\sup_c|\widehat F_{\rm cal}(c)-F_{\rm cal}(c)|,
    \qquad
    r_n=|q_n-q|,
    \label{eq:emp_errors}
\end{equation}
where $q_n$ is the finite-sample quantile level used to compute the empirical correction.

\begin{assumption}[Calibration-test score stability]
\label{asm:stability}
There exists $\rho_n\ge 0$ such that
\begin{equation*}
    |F_{\rm test}(c^*)-F_{\rm cal}(c^*)|\le\rho_n.
\end{equation*}
\end{assumption}

This assumption imposes calibration-test stability at the target multiplier
$c^*$ for the normalized retrieval error distribution. The term $\rho_n$
captures the residual mismatch remaining after retrieval-based normalization.

\begin{assumption}[Local identifiability of the calibration score quantile]
\label{asm:identifiability}
There exist constants $m>0$ and $\delta_{\rm cal}>0$ such that $F_{\rm cal}(c^*)=q$, and for every $u\in[0,\delta_{\rm cal}]$,
\begin{equation*}
\begin{aligned}
    F_{\rm cal}(c^*+u) &\ge q+mu, \\
    F_{\rm cal}(c^*-u) &\le q-mu.
\end{aligned}
\end{equation*}
\end{assumption}

This assumption keeps $\widehat c$ close to $c^*$, with $m$ controlling the local
separation of the calibration CDF near the target quantile.

\begin{assumption}[Local Lipschitzness of the test score CDF]
\label{asm:lipschitz}
There exist constants $L<\infty$ and $\delta_{\rm test}>0$ such that, for every $c$ satisfying $|c-c^*|\le\delta_{\rm test}$,
\begin{equation*}
    |F_{\rm test}(c)-F_{\rm test}(c^*)|\le L|c-c^*|.
\end{equation*}
\end{assumption}

Local smoothness of the test score CDF converts multiplier error near
$c^*$ into a test coverage-gap bound. The constant $L$ controls the local
sensitivity of test coverage to multiplier perturbations.

\begin{assumption}[Empirical calibration accuracy]
\label{asm:emp_acc}
The empirical calibration error and the finite-sample quantile-level error satisfy
\begin{equation*}
    e_n+r_n\le\frac{m\delta_0}{2}.
\end{equation*}
\end{assumption}

Here, $\delta_0=\min\{\delta_{cal}, \delta_{test}\}$ is the local radius around $c^*$ over which the regularity
assumptions hold. The condition ensures that calibration error is small enough
to keep $\widehat c$ within this neighborhood, with $e_n$ and $r_n$ capturing
empirical CDF error and finite-sample quantile adjustment, respectively.

\begin{lemma}[Score-event equivalence]
\label{lem:event_equiv}
For any $c\ge 0$,
\begin{equation*}
    Y_t\in C_t(c)\quad\Longleftrightarrow\quad B_t\le c.
\end{equation*}
\end{lemma}

This equivalence reduces coverage analysis to the scalar score distribution $F_{\rm test}$, and the full proof is provided in Appendix~\ref{app:theory}.

\begin{theorem}[Coverage gap bound of RCCP]
\label{thm:coverage_gap}
Suppose Assumptions~\ref{asm:stability}--\ref{asm:emp_acc} hold, and let
\begin{equation*}
    u_n:=\frac{2(e_n+r_n)}{m}.
\end{equation*}
If $u_n<\delta_0$, then
\begin{equation*}
    \bigl|
        \mathbb P_{\rm test}\{Y_t\in C_t(\widehat c)\}-q
    \bigr|
    \le
    \rho_n+\frac{2L}{m}(e_n+r_n).
\end{equation*}
\end{theorem}

The bound separates calibration-test score mismatch from empirical correction error. Retrieval affects the interval through the sharpness and stability of $\widetilde C_t$, while coverage is controlled through the normalized retrieval error $B_t$. The full proof is provided in Appendix~\ref{app:theory}.

\begin{corollary}[Asymptotic Coverage of RCCP]
\label{cor:asymptotic_coverage}
Suppose Assumptions~\ref{asm:stability}--\ref{asm:emp_acc} hold with fixed constants $m>0$, $L<\infty$, and $\delta_0>0$. If $\rho_n\to0$, $e_n\to0$, and $r_n\to0$, then
\begin{equation*}
    \left|
        \mathbb P_{\rm test}\{Y_t\in C_t(\widehat c)\}-q
    \right|\to 0.
\end{equation*}
\end{corollary}

The corollary follows directly from the coverage-gap bound: as the calibration-test score mismatch and empirical correction error vanish, $C_t(\widehat c)$ attains the target coverage asymptotically. The full proof is provided in Appendix~\ref{app:theory}.

\section{Experiments}

We evaluate RCCP on time-series conformal prediction benchmarks against representative split, adaptive, and localized conformal baselines. Section~\ref{sec:exp_setup} describes the experimental setup. Section~\ref{sec:main_results} reports the main benchmark results and further analyses, and Section~\ref{sec:ablations} presents ablation studies.

\subsection{Experimental Setup}
\label{sec:exp_setup}

\paragraph{Datasets.}
Our evaluation covers four benchmark datasets common in time series forecasting and conformal prediction studies, summarized in Table~\ref{tab:dataset_statistics}. For multi-location datasets, each location is treated as a separate series.
\vspace{4pt}

\begin{table}[!htbp]
\centering
\caption{Dataset statistics after preprocessing. Length denotes time points per series.}
\label{tab:dataset_statistics}
\resizebox{0.8\columnwidth}{!}{%
\begin{tabular}{llccc}
\toprule
Dataset & Frequency & Target & Length & \# Loc. \\
\midrule
Air~\cite{zhang2017cautionary} & Hourly & PM$_{10}$ & 35,064 & 12 \\
Solar~\cite{sengupta2018national} & Hourly & DHI & 26,304 & 50 \\
Electricity~\cite{harries1999splice} & 30-min & Transfer & 3,444 & 1 \\
Wind~\cite{enbpi} & Hourly & MWH & 13,870 & 1 \\
\bottomrule
\end{tabular}%
}
\end{table}

\definecolor{stdgray}{gray}{0.50}
\definecolor{bestbg}{gray}{0.90}    

\newcommand{\valfont}{\small}
\newcommand{\headfont}{\normalsize}
\newcolumntype{M}{>{\valfont}c}
\newcolumntype{N}{>{\valfont}c}

\newcommand{\ms}[2]{$#1$\,{\scriptsize\textcolor{stdgray}{$\pm\,#2$}}}
\newcommand{\gms}[2]{\textcolor{stdgray}{$#1$\,{\scriptsize$\pm\,#2$}}}
\newcommand{\bms}[2]{\cellcolor{bestbg}$\mathbf{#1}$\,{\scriptsize\textcolor{stdgray}{$\pm\,#2$}}}
\newcommand{\ums}[2]{$\underline{#1}$\,{\scriptsize\textcolor{stdgray}{$\pm\,#2$}}}

\begin{table*}[t]
\centering
\caption{
Performance across datasets and backbone forecasters at $\alpha=0.1$.
Values are mean $\pm$ standard deviation over 5 seeds after hyperparameter selection.
For $\Delta_{\mathrm{Cov}}$, values closer to zero are better, and lower values are better for PI-Width, Winkler, and Calibration Time.
Calibration Time measures conformal calibration and test interval construction after forecasts are available.
Shaded Winkler cells mark the lowest value, underlined Winkler values mark the second-lowest value, and shaded $\Delta_{\mathrm{Cov}}$ cells indicate severe undercoverage cases with $\Delta_{\mathrm{Cov}}<-2\%$.
}
\label{tab:main_results}

\small
\setlength{\tabcolsep}{1.8pt}
\renewcommand{\arraystretch}{1.00}

\begin{tabular*}{0.95\textwidth}{@{\extracolsep{\fill}}cc M *{7}{N}@{}}
\toprule
{\headfont Dataset}
& {\headfont Backbone}
& {\headfont Metric}
& {\headfont SCP}
& {\headfont EnbPI}
& {\headfont SPCI}
& {\headfont NexCP}
& {\headfont HopCPT}
& {\headfont ResCP}
& {\headfont RCCP (Ours)} \\
\midrule

\multirow[c]{8}{*}{Air}
& \multirow[c]{4}{*}{LSTM} & $\Delta$Cov & \ms{0.16}{0.22} & \gms{-6.36}{0.11} & \ms{-0.83}{1.10} & \ms{-1.06}{0.07} & \ms{-0.19}{0.17} & \ms{-0.14}{0.25} & \ms{0.43}{0.16} \\
& & PI-Width & \ms{62.72}{0.58} & \ms{53.39}{0.44} & \ms{63.09}{2.94} & \ms{63.72}{0.31} & \ms{62.47}{0.79} & \ms{61.35}{0.73} & \ms{60.20}{0.57} \\
& & Winkler & \ms{112.55}{0.63} & \ms{110.67}{0.61} & \ms{109.45}{0.34} & \ms{103.91}{0.21} & \ms{111.66}{0.93} & \ums{101.57}{0.48} & \bms{92.44}{0.35} \\
& & Time & \ms{35.28}{0.26} & \ms{143.38}{0.80} & \ms{664.91}{2.89} & \ms{35.52}{0.45} & \ms{262.74}{0.54} & \ms{102.75}{6.07} & \ms{70.94}{0.55} \\
\cmidrule(lr){2-10}

& \multirow[c]{4}{*}{Transf} & $\Delta$Cov & \ms{0.37}{0.15} & \gms{-6.43}{0.16} & \ms{-0.51}{0.66} & \ms{-1.69}{0.08} & \ms{-0.46}{0.39} & \ms{-0.28}{0.11} & \ms{0.20}{0.09} \\
& & PI-Width & \ms{63.88}{0.90} & \ms{53.33}{1.08} & \ms{64.72}{1.48} & \ms{65.43}{1.02} & \ms{62.04}{1.15} & \ms{61.45}{1.45} & \ms{60.94}{1.22} \\
& & Winkler & \ms{116.96}{2.75} & \ms{111.67}{3.05} & \ms{109.64}{2.57} & \ms{102.71}{1.50} & \ms{115.88}{3.29} & \ums{102.21}{1.82} & \bms{92.51}{1.44} \\
& & Time & \ms{35.34}{1.55} & \ms{143.78}{1.23} & \ms{658.39}{7.32} & \ms{35.14}{1.22} & \ms{247.23}{1.89} & \ms{98.77}{0.90} & \ms{72.45}{0.50} \\

\midrule

\multirow[c]{8}{*}{Solar}
& \multirow[c]{4}{*}{LSTM} & $\Delta$Cov & \ms{1.50}{0.09} & \ms{-1.84}{0.06} & \ms{0.75}{0.76} & \ms{-0.79}{0.08} & \ms{1.26}{2.20} & \ms{1.14}{0.14} & \ms{0.42}{0.24} \\
& & PI-Width & \ms{35.24}{1.36} & \ms{22.43}{1.04} & \ms{27.42}{3.07} & \ms{26.18}{0.90} & \ms{29.37}{8.77} & \ms{20.10}{1.28} & \ms{18.31}{0.38} \\
& & Winkler & \ms{62.87}{2.09} & \ms{55.08}{1.95} & \ms{54.86}{1.71} & \ms{51.19}{1.60} & \ms{39.02}{7.87} & \ums{38.90}{1.43} & \bms{26.08}{0.67} \\
& & Time & \ms{90.14}{1.11} & \ms{443.45}{5.51} & \ms{1964.26}{19.29} & \ms{96.13}{0.42} & \ms{862.02}{111.56} & \ms{323.09}{1.10} & \ms{167.27}{0.90} \\
\cmidrule(lr){2-10}

& \multirow[c]{4}{*}{Transf} & $\Delta$Cov & \ms{1.14}{0.06} & \ms{-1.61}{0.16} & \ms{1.48}{0.30} & \ms{-1.02}{0.31} & \ms{0.39}{0.81} & \ms{1.94}{0.48} & \ms{1.77}{0.52} \\
& & PI-Width & \ms{31.39}{2.30} & \ms{23.96}{1.80} & \ms{24.19}{1.88} & \ms{28.87}{1.62} & \ms{28.37}{4.54} & \ms{20.69}{0.78} & \ms{21.02}{1.05} \\
& & Winkler & \ms{58.85}{1.94} & \ms{54.76}{1.11} & \ms{44.57}{3.89} & \ms{52.26}{1.53} & \ms{43.27}{8.91} & \ums{38.41}{0.43} & \bms{29.29}{1.70} \\
& & Time & \ms{89.40}{2.39} & \ms{433.94}{3.78} & \ms{1925.61}{25.64} & \ms{94.98}{1.89} & \ms{764.09}{58.11} & \ms{317.41}{10.46} & \ms{186.28}{1.12} \\

\midrule

\multirow[c]{8}{*}{Wind}
& \multirow[c]{4}{*}{LSTM} & $\Delta$Cov & \ms{0.23}{0.15} & \gms{-2.31}{0.10} & \ms{-1.95}{2.81} & \ms{-0.06}{0.05} & \ms{2.25}{2.04} & \ms{1.56}{0.15} & \ms{0.62}{0.49} \\
& & PI-Width & \ms{63.42}{0.66} & \ms{59.54}{0.36} & \ms{58.34}{8.97} & \ms{62.60}{0.49} & \ms{73.20}{10.36} & \ms{63.36}{0.55} & \ms{58.75}{1.11} \\
& & Winkler & \ms{96.29}{0.09} & \ms{97.92}{0.33} & \ms{99.55}{1.73} & \ms{94.50}{0.25} & \ms{99.06}{3.03} & \ums{89.75}{0.47} & \bms{83.79}{0.43} \\
& & Time & \ms{0.72}{0.02} & \ms{4.49}{0.07} & \ms{19.32}{0.34} & \ms{0.86}{0.02} & \ms{8.89}{0.20} & \ms{4.05}{0.01} & \ms{1.50}{0.01} \\
\cmidrule(lr){2-10}

& \multirow[c]{4}{*}{Transf} & $\Delta$Cov & \ms{-0.50}{0.81} & \gms{-2.06}{0.14} & \ms{-0.40}{1.84} & \ms{0.37}{0.18} & \ms{-0.64}{3.77} & \ms{1.49}{0.42} & \ms{1.52}{0.15} \\
& & PI-Width & \ms{65.28}{1.49} & \ms{60.86}{1.62} & \ms{64.50}{6.10} & \ms{67.65}{2.65} & \ms{66.32}{6.82} & \ms{64.83}{1.81} & \ms{63.52}{1.41} \\
& & Winkler & \ms{101.35}{2.17} & \ms{100.25}{1.59} & \ms{100.79}{1.87} & \ms{100.68}{1.60} & \ms{99.09}{3.22} & \ums{93.41}{2.33} & \bms{87.78}{1.02} \\
& & Time & \ms{0.71}{0.01} & \ms{4.50}{0.13} & \ms{19.03}{0.37} & \ms{0.84}{0.03} & \ms{8.62}{0.53} & \ms{3.87}{0.08} & \ms{1.52}{0.01} \\

\midrule

\multirow[c]{8}{*}{Electricity}
& \multirow[c]{4}{*}{LSTM} & $\Delta$Cov & \ms{5.18}{1.29} & \gms{-3.12}{0.19} & \ms{4.13}{0.94} & \ms{1.09}{0.91} & \ms{4.28}{1.52} & \ms{1.23}{0.19} & \ms{0.54}{0.72} \\
& & PI-Width & \ms{0.21}{0.02} & \ms{0.15}{0.01} & \ms{0.19}{0.02} & \ms{0.17}{0.01} & \ms{0.20}{0.02} & \ms{0.16}{0.01} & \ms{0.16}{0.01} \\
& & Winkler & \ms{0.25}{0.02} & \ms{0.24}{0.02} & \ms{0.24}{0.02} & \ums{0.23}{0.02} & \ms{0.25}{0.02} & \ums{0.23}{0.01} & \bms{0.22}{0.01} \\
& & Time & \ms{0.20}{0.03} & \ms{1.18}{0.04} & \ms{4.41}{0.03} & \ms{0.24}{0.00} & \ms{4.36}{0.65} & \ms{1.77}{0.20} & \ms{0.23}{0.00} \\
\cmidrule(lr){2-10}

& \multirow[c]{4}{*}{Transf} & $\Delta$Cov & \ms{5.33}{3.12} & \gms{-5.98}{0.65} & \ms{-0.11}{2.30} & \ms{2.61}{1.14} & \ms{4.75}{0.36} & \ms{1.96}{0.56} & \ms{1.99}{1.61} \\
& & PI-Width & \ms{0.30}{0.03} & \ms{0.16}{0.01} & \ms{0.19}{0.02} & \ms{0.25}{0.02} & \ms{0.28}{0.02} & \ms{0.20}{0.01} & \ms{0.20}{0.03} \\
& & Winkler & \ms{0.35}{0.02} & \ms{0.31}{0.03} & \ums{0.28}{0.01} & \ms{0.32}{0.04} & \ms{0.34}{0.03} & \ms{0.29}{0.02} & \bms{0.27}{0.02} \\
& & Time & \ms{0.18}{0.00} & \ms{1.11}{0.01} & \ms{4.32}{0.12} & \ms{0.23}{0.01} & \ms{3.77}{0.13} & \ms{1.87}{0.63} & \ms{0.27}{0.01} \\

\bottomrule
\end{tabular*}

\vspace{-1.5mm}
\end{table*}

\paragraph{Time series forecasters.}
We evaluate each conformal method with LSTM and Transformer backbones. For each dataset, observations are chronologically split into training, calibration, and test periods with proportions 40\%, 40\%, and 20\%, respectively. We train one global forecaster per dataset on the training period and keep it fixed during calibration and testing~\cite{hopcpt}. Calibration observations are used only for conformal calibration, and all conformal methods are evaluated on the same frozen forecasts. Unless otherwise stated, both backbones use hidden dimension 128 and batch size 256, and are trained for 150 epochs with Adam using learning rate $10^{-3}$.

\paragraph{Baselines.}
We compare RCCP with representative conformal baselines for time series. The baselines include split conformal prediction (SCP)~\cite{vovk2005algorithmic}, online or adaptive methods such as EnbPI~\cite{enbpi} and SPCI~\cite{spci}, and localized methods such as NexCP~\cite{nexcp}, HopCPT~\cite{hopcpt}, and ResCP~\cite{rescp}. All methods use the same temporal partitions, frozen point forecasts, and random seeds, with hyperparameters selected on the calibration split when needed.

\paragraph{Evaluation metrics.}
For each test point, let $[L_t,U_t]$ denote the reported prediction interval. We report empirical coverage, signed coverage gap, PI-Width, Winkler score, and Calibration Time. Empirical coverage measures validity,
\begin{equation}
\widehat{\mathrm{Cov}}
=
\frac{1}{|\mathcal I_{\mathrm{test}}|}
\sum_{t\in\mathcal I_{\mathrm{test}}}
\mathbf 1\{Y_t\in [L_t,U_t]\}.
\label{eq:metric_coverage}
\end{equation}
The signed coverage gap is reported in percentage points,
\begin{equation}
\Delta_{\mathrm{Cov}}
=
100\left(\widehat{\mathrm{Cov}}-(1-\alpha)\right),
\label{eq:metric_gap}
\end{equation}
so negative values indicate undercoverage. PI-Width measures average interval length,
\begin{equation}
\mathrm{PI\mbox{-}Width}
=
\frac{1}{|\mathcal I_{\mathrm{test}}|}
\sum_{t\in\mathcal I_{\mathrm{test}}}(U_t-L_t).
\label{eq:metric_width}
\end{equation}
To jointly assess interval length and miscoverage, we report the Winkler score,
\begin{equation}
W_t=(U_t-L_t)
+\frac{2}{\alpha}(L_t-Y_t)\mathbf 1\{Y_t<L_t\}
+\frac{2}{\alpha}(Y_t-U_t)\mathbf 1\{Y_t>U_t\},
\label{eq:metric_winkler_single}
\end{equation}
and average it over the test period. Calibration Time is the post-forecasting time needed to calibrate each conformal method and construct test intervals. For RCCP, it also includes serialization of retrieval keys for knowledge-base construction.

\begin{figure}[t]
    \centering

    \includegraphics[width=0.5\linewidth]{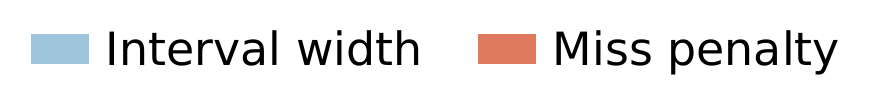}
    \vspace{-0.3em}

    \begin{minipage}[t]{0.485\linewidth}
        \centering
        \includegraphics[width=\linewidth]{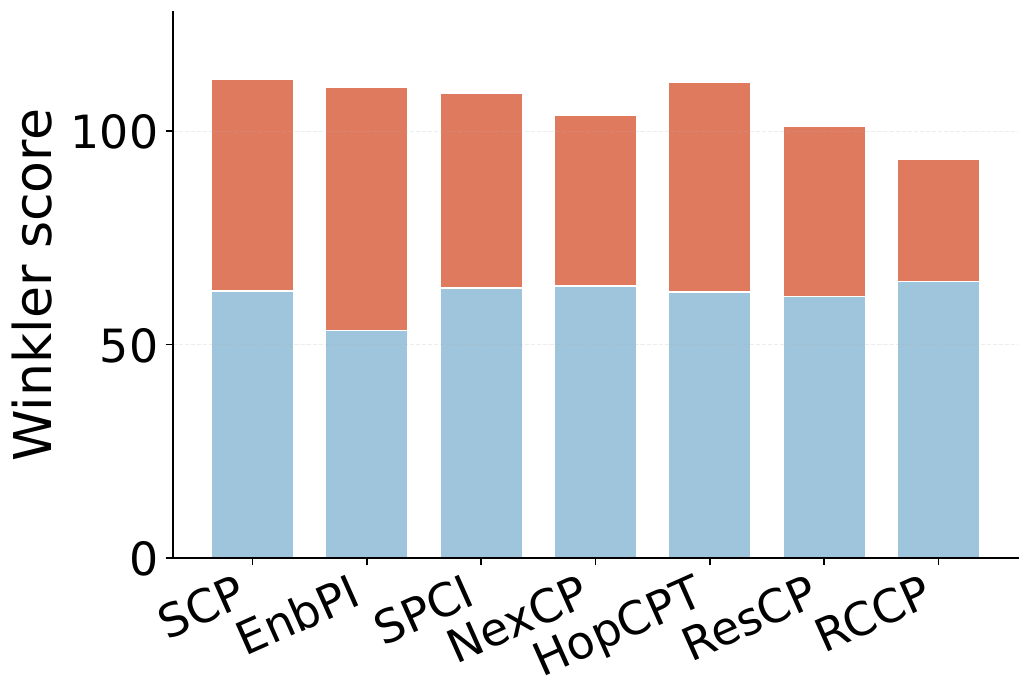}
        \vspace{-1.0em}
        \caption*{\small\textbf{(a) Air}}
    \end{minipage}
    \hfill
    \begin{minipage}[t]{0.485\linewidth}
        \centering
        \includegraphics[width=\linewidth]{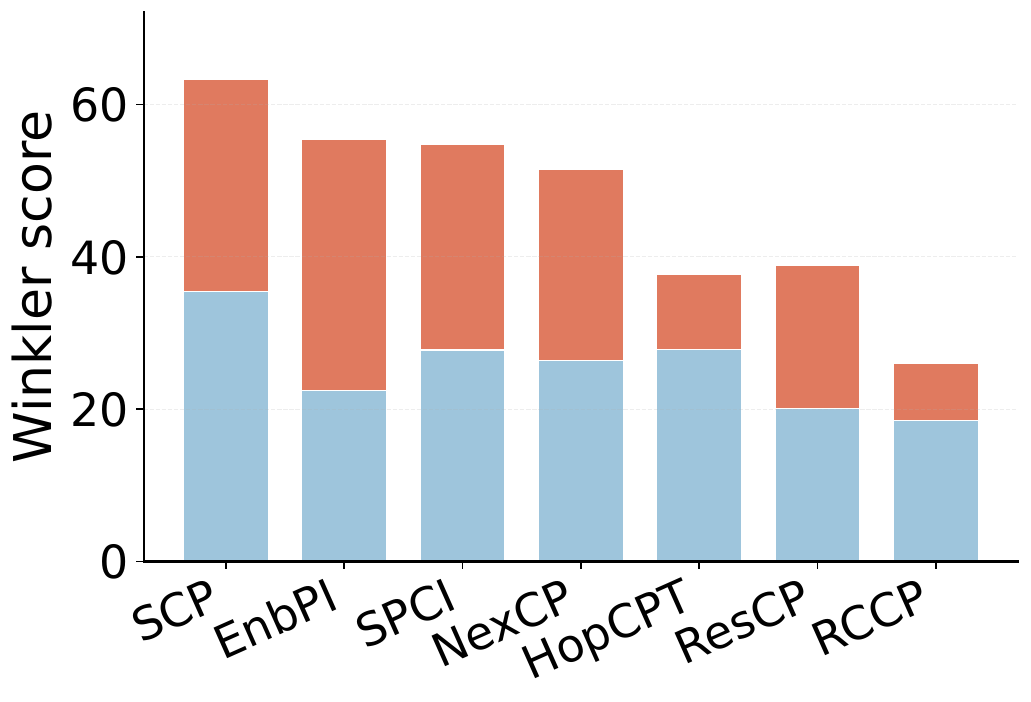}
        \vspace{-1.0em}
        \caption*{\small\textbf{(b) Solar}}
    \end{minipage}

    \vspace{-0em}
    \caption{Winkler-score decomposition into interval width and miss penalty on Air and Solar.}
    \label{fig:winkler_decomp}
\end{figure}

\begin{figure*}[t]
    \centering

    \includegraphics[width=0.65\textwidth]{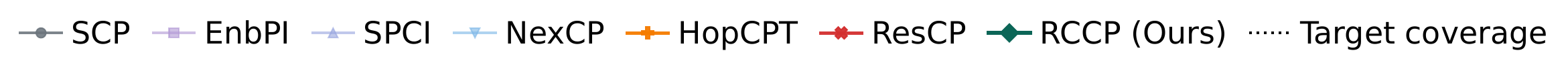}

    \begin{minipage}[t]{0.47\textwidth}
        \centering
        \includegraphics[width=\linewidth]{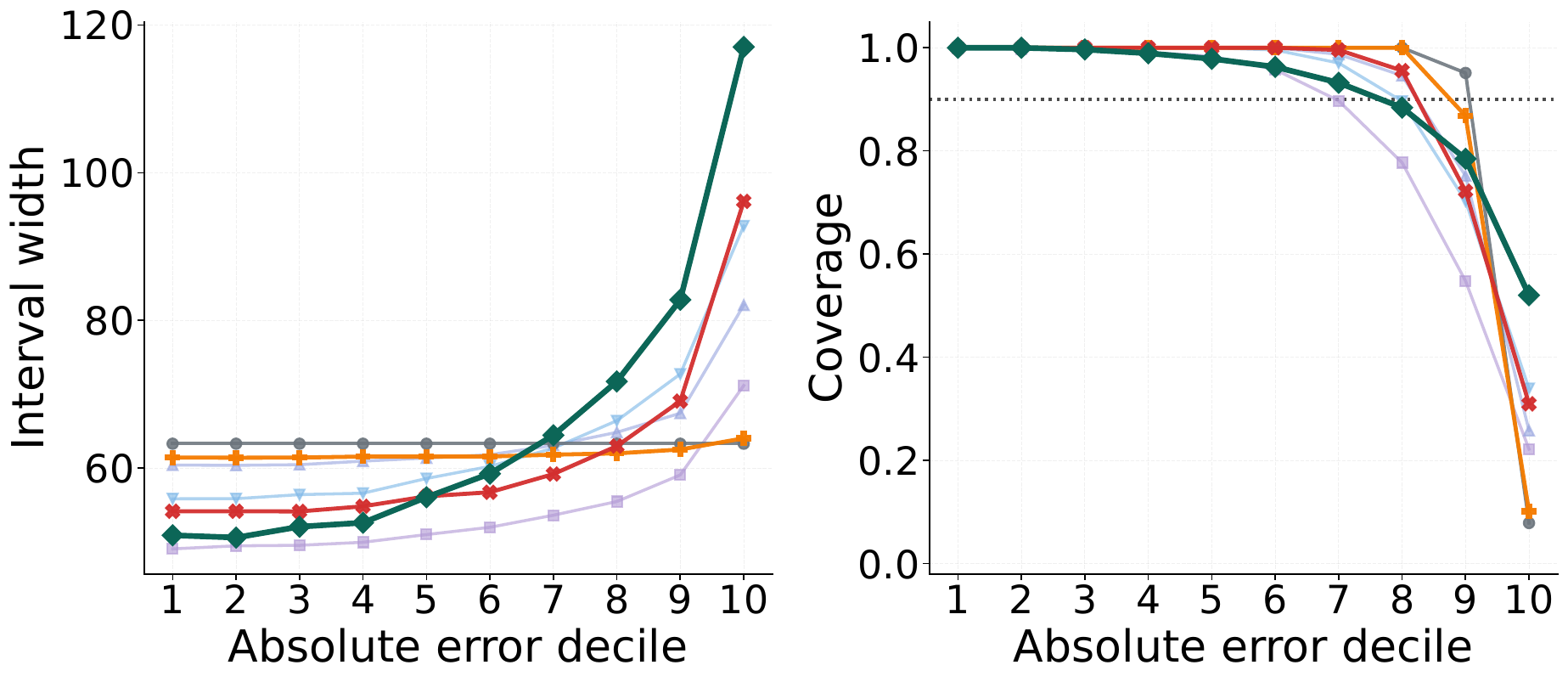}
        \vspace{-1.0em}
        
        {\small\textbf{(a) Air}}
    \end{minipage}
    \hfill
    \begin{minipage}[t]{0.47\textwidth}
        \centering
        \includegraphics[width=\linewidth]{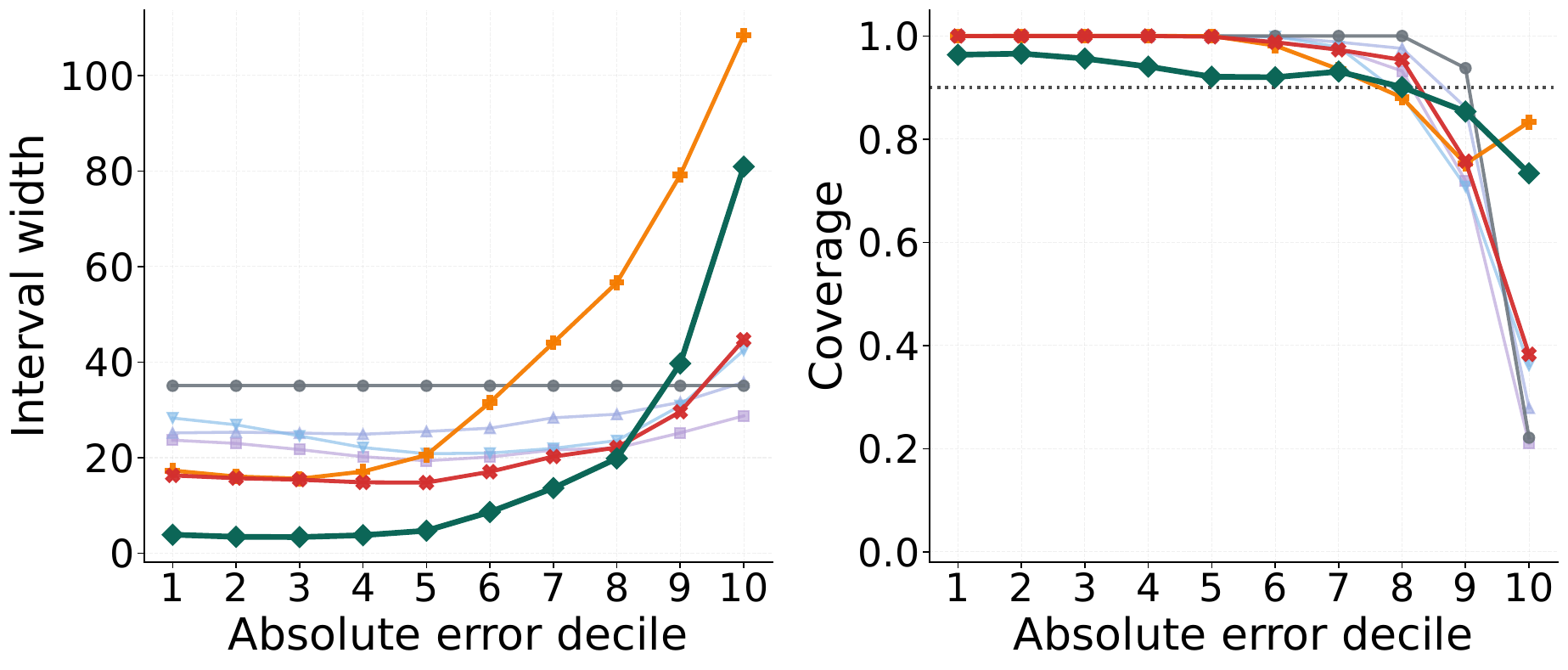}
        \vspace{-1.0em}
        
        {\small\textbf{(b) Solar}}
    \end{minipage}

    \caption{Interval width and empirical coverage stratified by realized error decile. Higher deciles correspond to larger realized errors and test whether intervals adapt to difficult points.}
    \label{fig:error_decile}
    \vspace{-5pt}
\end{figure*}

\subsection{Results \& Analysis}
\label{sec:main_results}

Across all datasets and backbone forecasters, Table~\ref{tab:main_results} shows that RCCP avoids empirical undercoverage at $\alpha=0.1$ and achieves the best or tied-best Winkler scores. Computationally simple baselines such as SCP and NexCP have low Calibration Time because they rely on global or weighted calibration rules, but this simplicity often comes with weaker coverage or larger Winkler scores. EnbPI and SPCI adapt to recent residuals, yet their intervals remain less effective when the error scale changes over time. HopCPT is closer in spirit to RCCP because it uses learned local information, but it requires substantially larger Calibration Time and shows variable interval quality across settings. These results support the main design choice of RCCP, where retrieval improves local interval allocation and correction prevents the retrieved interval from becoming an uncalibrated local threshold.



Calibration Time should be interpreted conservatively, since it is measured after forecasts are available and includes method-specific calibration and test interval construction. For RCCP, it also includes serializing the forecast-state keys used for retrieval, so the reported time includes retrieval memory construction. Although SCP and NexCP are computationally cheaper by construction, RCCP is more efficient than more expensive adaptive or learned local baselines while delivering stronger coverage and Winkler performance. Results for additional miscoverage levels are reported in Appendix~\ref{app:add_results}.

Winkler reflects interval length and coverage violations, so the analysis first decomposes this score. A lower Winkler score can result from shorter intervals, smaller violations, or both. Figure~\ref{fig:winkler_decomp} separates interval width from the penalty incurred by miscovered points. RCCP's advantage does not come from producing the narrowest intervals, but from reducing the miscoverage penalty. This indicates that targets outside the interval tend to be less severe violations, consistent with the correction mechanism that adjusts the retrieved interval scale without uniformly inflating intervals.

\vspace{0.5em}
\begin{table}[h]
\centering
\caption{Width adaptivity ratio across datasets. Values are Decile 10 width divided by Decile 1 width.}
\label{tab:width_adaptivity}
\scriptsize
\setlength{\tabcolsep}{1.7pt}
\renewcommand{\arraystretch}{1.08}
\resizebox{0.85\columnwidth}{!}{%
\begin{tabular}{@{}lccccccc@{}}
\toprule
Dataset & SCP & EnbPI & SPCI & NexCP & HopCPT & ResCP & RCCP \\
\midrule
Air & 1.00 & 1.45 & 1.36 & 1.66 & 1.04 & 1.77 & 2.30 \\
Solar & 1.00 & 1.22 & 1.42 & 1.50 & 6.27 & 2.74 & 20.86 \\
\bottomrule
\end{tabular}%
}
\end{table}
\vspace{4pt}

The decile analysis explains how the penalty is reduced. Figure~\ref{fig:error_decile} stratifies test points by realized error decile and compares interval width with empirical coverage. RCCP widens intervals in high-error deciles and maintains stronger coverage where prediction errors are largest. This pattern is summarized by the width adaptivity ratio in Table~\ref{tab:width_adaptivity}, defined as the Decile 10 width divided by the Decile 1 width. RCCP reaches 2.30 on Air and 20.86 on Solar, exceeding all baselines. Thus, RCCP reallocates interval width toward high-error regions rather than widening intervals uniformly.

\begin{figure}[h]
    \centering
    \includegraphics[width=0.95\linewidth]{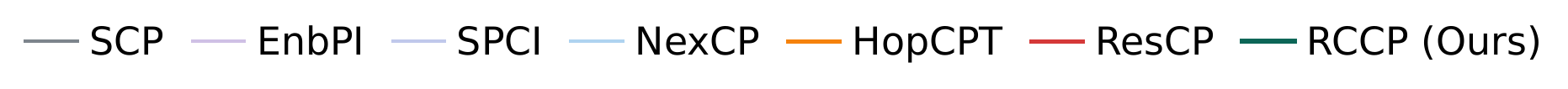}
    \vspace{-1pt}

    \begin{minipage}[t]{0.48\linewidth}
        \centering
        \includegraphics[width=\linewidth]{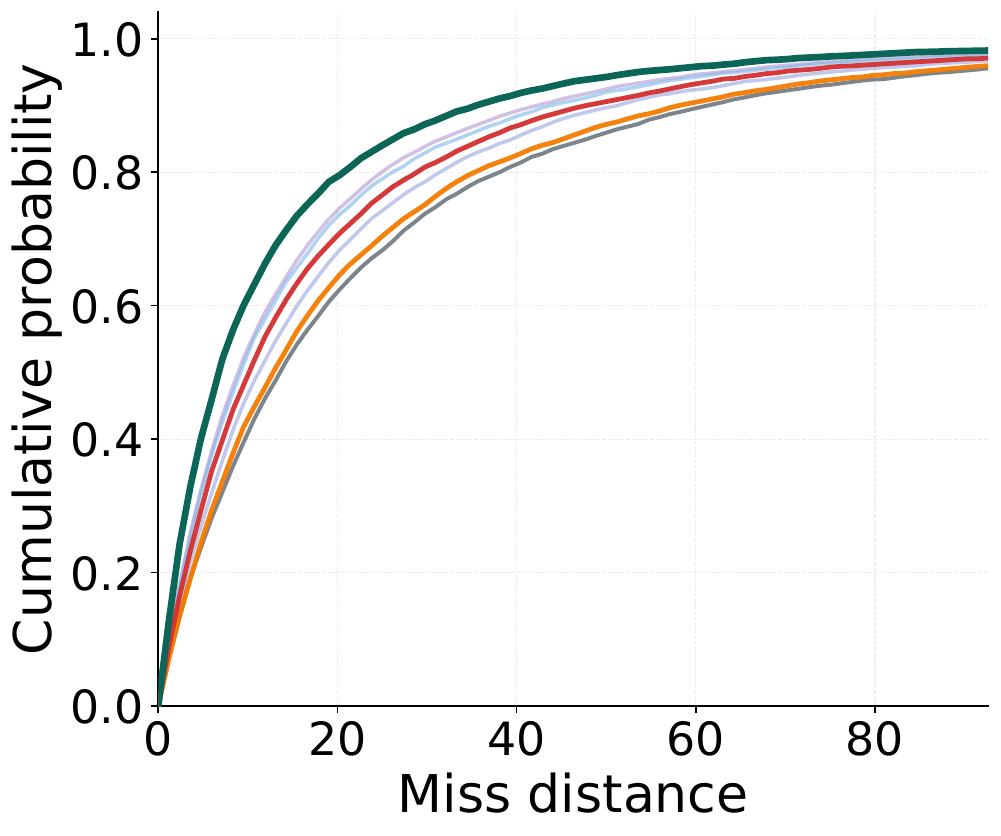}
        \vspace{-1.0em}
        \caption*{\small\textbf{(a) Air}}
    \end{minipage}
    \hfill
    \begin{minipage}[t]{0.48\linewidth}
        \centering
        \includegraphics[width=\linewidth]{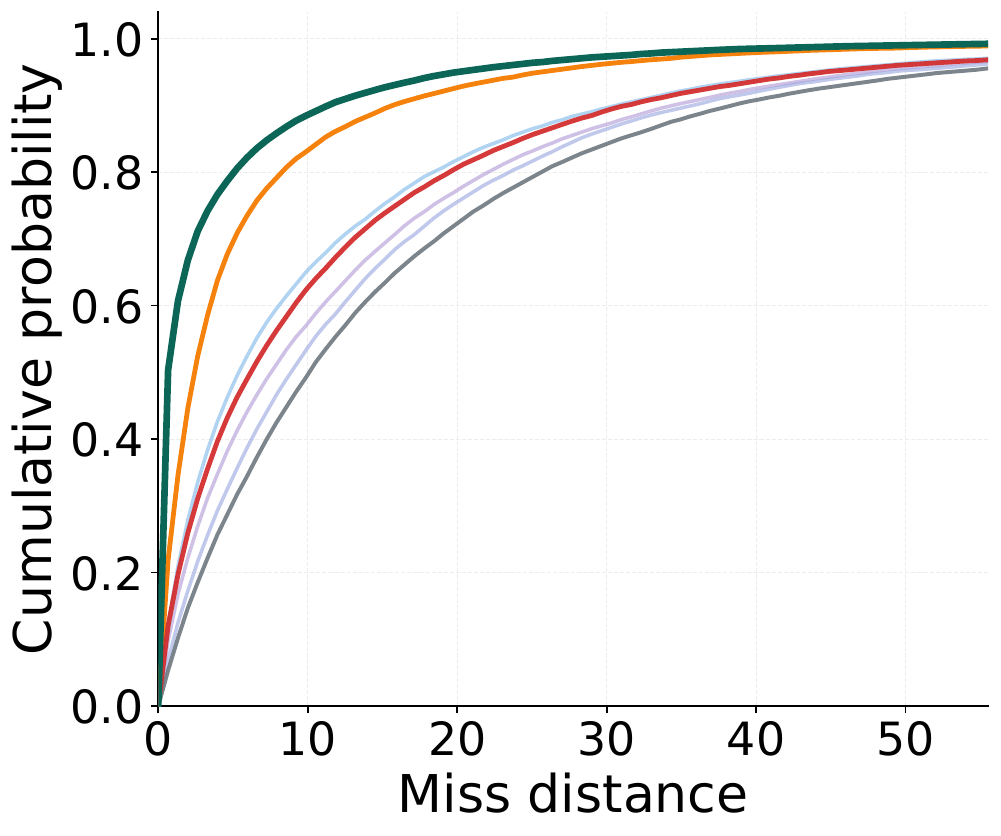}
        \vspace{-1.0em}
        \caption*{\small\textbf{(b) Solar}}
    \end{minipage}

    \vspace{-0em}
    \caption{CDF of miscoverage distance for uncovered test points. Higher curves indicate smaller coverage violations.}
    \label{fig:miss_distance}
\end{figure}

The distribution of violation sizes further supports this interpretation. Figure~\ref{fig:miss_distance} plots the CDF of the distance from each miscovered target to the nearest interval boundary, where curves concentrated near zero indicate smaller violations and heavier tails indicate severe misses. RCCP rises fastest across datasets, showing that even when coverage fails, the realized value usually remains close to the reported interval.

\begin{figure*}[t]
    \centering
    \includegraphics[width=0.32\textwidth]{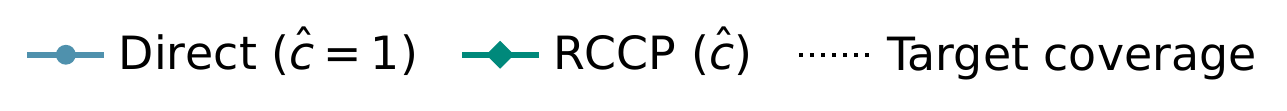}
    \vspace{-1mm}

    \begin{subfigure}[t]{0.49\textwidth}
        \centering
        \includegraphics[width=\linewidth]{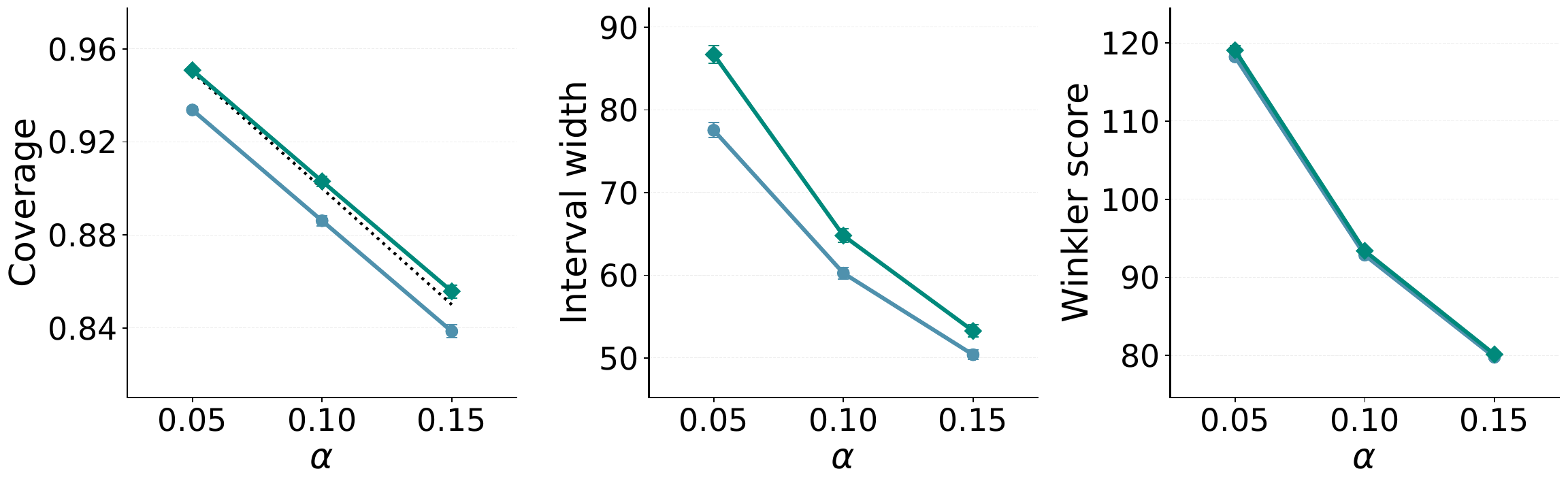}
        \caption{Air}
        \label{fig:correct_air}
    \end{subfigure}
    \hfill
    \begin{subfigure}[t]{0.49\textwidth}
        \centering
        \includegraphics[width=\linewidth]{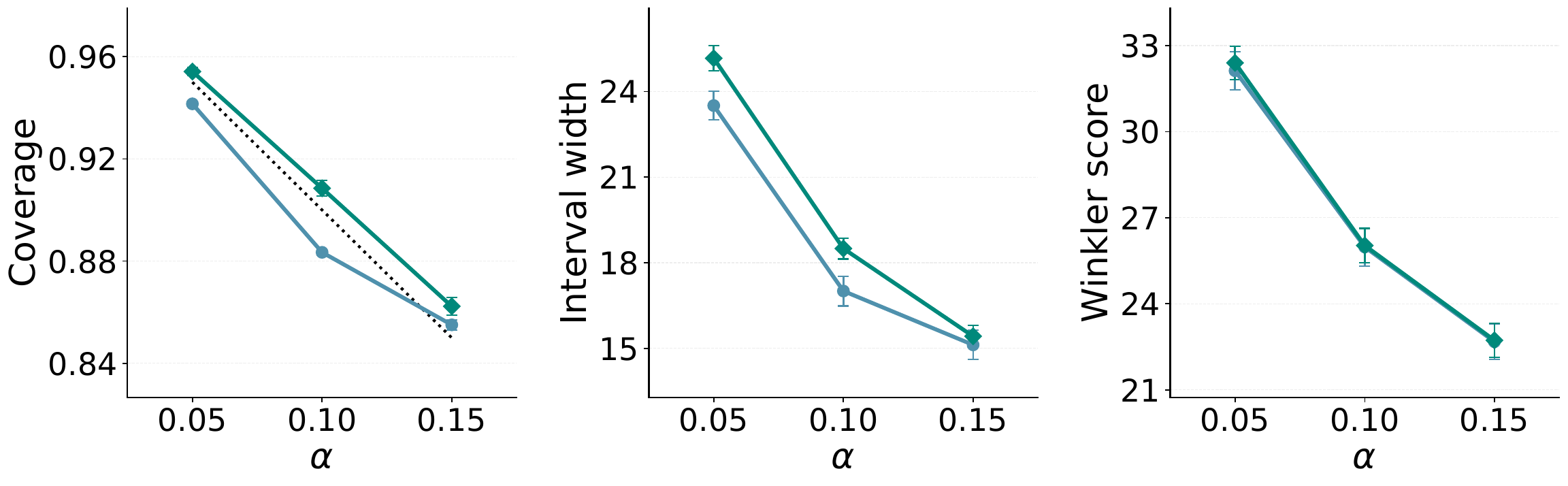}
        \caption{Solar}
        \label{fig:correct_solar}
    \end{subfigure}
    \caption{Coverage, interval width, and Winkler score for direct retrieval and corrected intervals across target miscoverage levels. The dotted line marks target coverage, showing how correction affects calibration and interval efficiency as $\alpha$ varies.}
    \label{fig:correct_alpha}
\end{figure*}

\subsection{Ablation Studies}
\label{sec:ablations}

The ablation study examines the interval components, correction across target miscoverage levels, retrieval key construction and metric, and retrieval size. These analyses clarify which design choices affect coverage, interval efficiency, and robustness.

\begin{table}[h]
\centering
\caption{Component ablation. The table isolates how asymmetric residuals and correction affect coverage and Winkler score.}
\label{tab:component_ablation}

\scriptsize
\setlength{\tabcolsep}{1.7pt}
\renewcommand{\arraystretch}{1.08}

\resizebox{0.9\columnwidth}{!}{%
\begin{tabular}{@{}lllcccc@{}}
\toprule
Dataset & Backbone & Metric
& Org. RCCP
& w/o Asym.
& w/o Corr.
& w/o Both \tabularnewline
\midrule

\multirow{4}{*}{Air}
& \multirow{2}{*}{LSTM}
& $\Delta$Cov
& $\mathbf{+0.43{\pm}0.16}$
& $+0.44{\pm}0.13$
& $-1.35{\pm}0.23$
& $-0.89{\pm}0.13$ \tabularnewline
& & Winkler
& $\mathbf{92.44{\pm}0.35}$
& $96.42{\pm}0.21$
& $92.91{\pm}0.32$
& $95.95{\pm}0.20$ \tabularnewline
\cmidrule(lr){2-7}

& \multirow{2}{*}{Transf.}
& $\Delta$Cov
& $\mathbf{+0.20{\pm}0.09}$
& $+0.49{\pm}0.23$
& $-3.01{\pm}0.10$
& $-2.60{\pm}0.20$ \tabularnewline
& & Winkler
& $\mathbf{92.51{\pm}1.44}$
& $97.26{\pm}1.04$
& $93.16{\pm}1.27$
& $97.63{\pm}1.02$ \tabularnewline

\specialrule{0.35pt}{0.8mm}{0.8mm}

\multirow{4}{*}{Solar}
& \multirow{2}{*}{LSTM}
& $\Delta$Cov
& $+0.42{\pm}0.24$
& $\mathbf{+0.03{\pm}0.22}$
& $-1.66{\pm}0.13$
& $-3.10{\pm}0.62$ \tabularnewline
& & Winkler
& $26.08{\pm}0.67$
& $28.17{\pm}1.35$
& $\mathbf{26.03{\pm}0.75}$
& $28.05{\pm}1.38$ \tabularnewline
\cmidrule(lr){2-7}

& \multirow{2}{*}{Transf.}
& $\Delta$Cov
& $+1.77{\pm}0.52$
& $+1.31{\pm}0.62$
& $\mathbf{+0.51{\pm}0.51}$
& $-2.31{\pm}0.62$ \tabularnewline
& & Winkler
& $\mathbf{29.29{\pm}1.70}$
& $35.93{\pm}3.94$
& $29.56{\pm}1.80$
& $35.85{\pm}3.98$ \tabularnewline

\bottomrule
\end{tabular}%
}
\end{table}

\paragraph{Component-wise ablation.}
Table~\ref{tab:component_ablation} separates the roles of the two design
choices introduced in Sections~\ref{sec:interval_retreival}
and~\ref{sec:retrieval_correction}: asymmetric residual scaling and conformal
correction. First, conformal correction mainly controls coverage. Removing it
worsens $\Delta$Cov and leads to undercoverage in most settings, since locally
retrieved residuals need not match the test-time normalized error distribution.
Second, asymmetric residual scaling mainly controls efficiency. Removing it
keeps coverage closer to the target but increases Winkler scores, because a
symmetric interval cannot adapt to directional error imbalance. One-sided
residuals avoid this waste by shaping the upper and lower widths separately.
Therefore, the results justify the full design: correction ensures reliable
coverage, while asymmetry improves interval efficiency.




\paragraph{Correction across target levels.}
The correction factor is evaluated across multiple target miscoverage levels. This comparison tests whether retrieval needs calibration after it has produced a locally adaptive interval. Figure~\ref{fig:correct_alpha} compares direct retrieval, which fixes $\widehat c=1$, with the corrected RCCP interval at $\alpha\in\{0.05,0.10,0.15\}$. Across target levels, correction moves empirical coverage toward the desired level with only modest changes in interval width and Winkler score. This indicates that the correction repairs residual coverage error left by retrieval rather than simply inflating intervals.

\vspace{4pt}
\begin{table}[h]
\centering
\caption{Retrieval key and metric ablation. The table compares embedding keys and raw-window keys across distance metrics.}

\label{tab:key_ablation}
\scriptsize
\setlength{\tabcolsep}{1.7pt}
\renewcommand{\arraystretch}{1.08}
\resizebox{0.9\columnwidth}{!}{%
\begin{tabular}{@{}lllcccc@{}}
\toprule
Dataset & Backbone & Metric
& Org. RCCP
& Emb.-Cos
& RawTS-L2
& RawTS-Cos \tabularnewline
\midrule

\multirow{4}{*}{Air}
& \multirow{2}{*}{LSTM}
& $\Delta$Cov
& $+0.43{\pm}0.16$
& $+0.37{\pm}0.18$
& $+0.21{\pm}0.21$
& $\mathbf{+0.08{\pm}0.22}$ \tabularnewline
& & Winkler
& $\mathbf{92.44{\pm}0.35}$
& $93.47{\pm}0.39$
& $96.30{\pm}0.77$
& $96.64{\pm}0.66$ \tabularnewline
\cmidrule(lr){2-7}

& \multirow{2}{*}{Transf.}
& $\Delta$Cov
& $+0.20{\pm}0.09$
& $+0.31{\pm}0.15$
& $\mathbf{-0.09{\pm}0.11}$
& $-0.22{\pm}0.11$ \tabularnewline
& & Winkler
& $\mathbf{92.51{\pm}1.44}$
& $92.62{\pm}1.30$
& $95.96{\pm}1.20$
& $97.40{\pm}1.63$ \tabularnewline

\specialrule{0.35pt}{0.8mm}{0.8mm}

\multirow{4}{*}{Solar}
& \multirow{2}{*}{LSTM}
& $\Delta$Cov
& $+0.42{\pm}0.24$
& $+0.75{\pm}0.23$
& $\mathbf{-0.03{\pm}0.19}$
& $+0.27{\pm}0.09$ \tabularnewline
& & Winkler
& $26.08{\pm}0.67$
& $\mathbf{26.06{\pm}0.69}$
& $28.42{\pm}0.76$
& $28.75{\pm}0.71$ \tabularnewline
\cmidrule(lr){2-7}

& \multirow{2}{*}{Transf.}
& $\Delta$Cov
& $+1.77{\pm}0.52$
& $+1.64{\pm}0.58$
& $\mathbf{+0.70{\pm}0.25}$
& $+0.89{\pm}0.26$ \tabularnewline
& & Winkler
& $29.29{\pm}1.70$
& $29.54{\pm}1.80$
& $\mathbf{29.05{\pm}0.84}$
& $29.49{\pm}0.81$ \tabularnewline

\bottomrule
\end{tabular}%
}
\end{table}

\paragraph{Retrieval key construction and metric.}
Table~\ref{tab:key_ablation} tests whether performance depends on retrieval key construction and the distance metric, separating retrieval choices from the correction step. The main RCCP setting is compared with embedding or raw-window keys under Euclidean distance or cosine similarity. Coverage remains stable across variants, suggesting that conformal correction is robust to retrieval choices. Winkler scores are more sensitive to key construction, with embedding keys generally yielding sharper corrected intervals.



\begin{figure}[h]
    \centering
    \includegraphics[
        width=0.35\linewidth
    ]{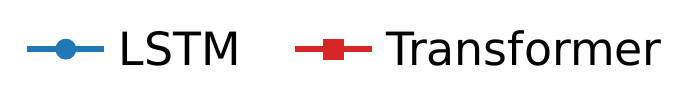}
    \vspace{-2pt}

    \begin{subfigure}[t]{0.48\linewidth}
        \centering
        \includegraphics[width=\linewidth]{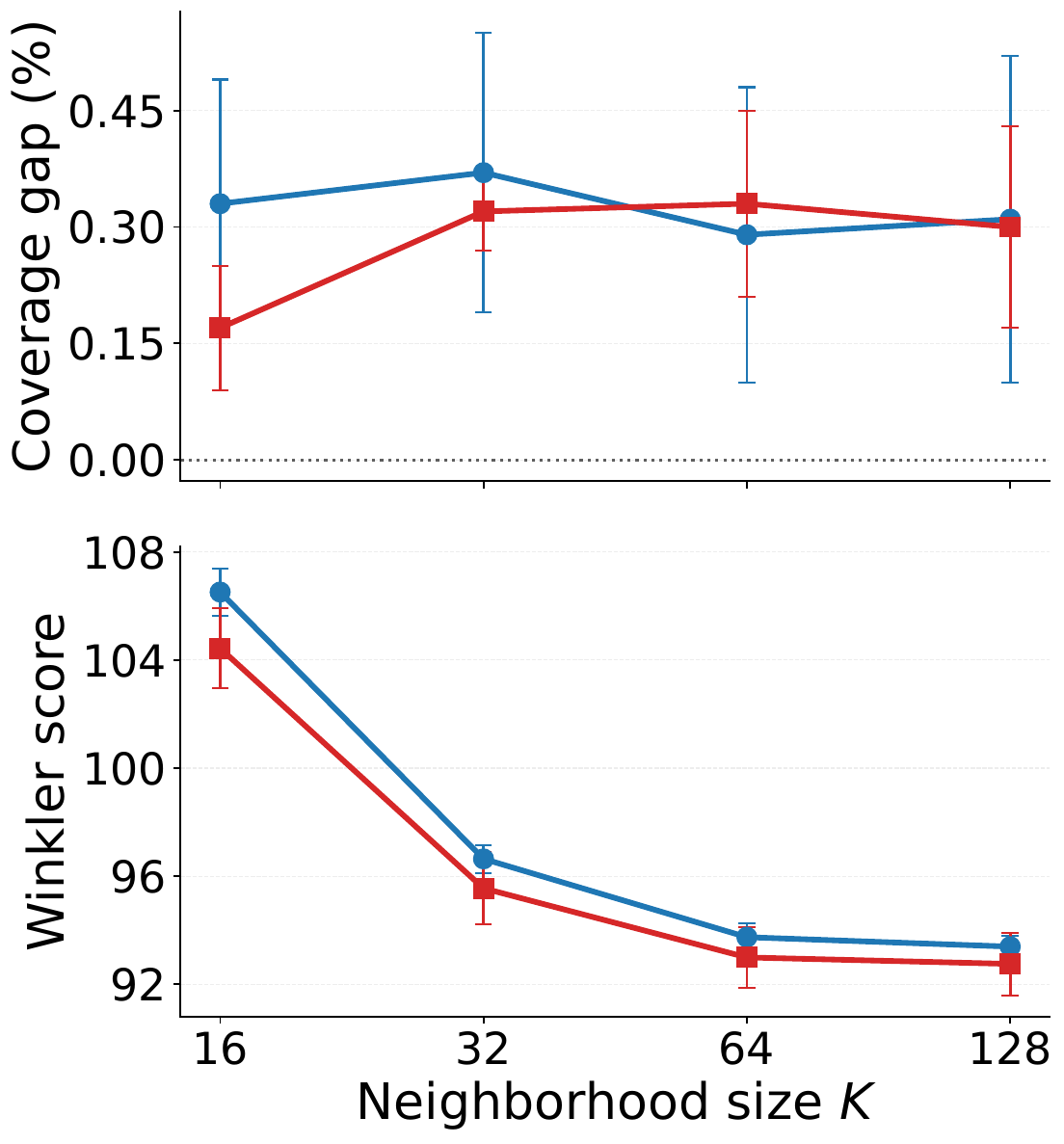}
        \caption{Air}
        \label{fig:k_sens_air}
    \end{subfigure}
    \hfill
    \begin{subfigure}[t]{0.48\linewidth}
        \centering
        \includegraphics[width=\linewidth]{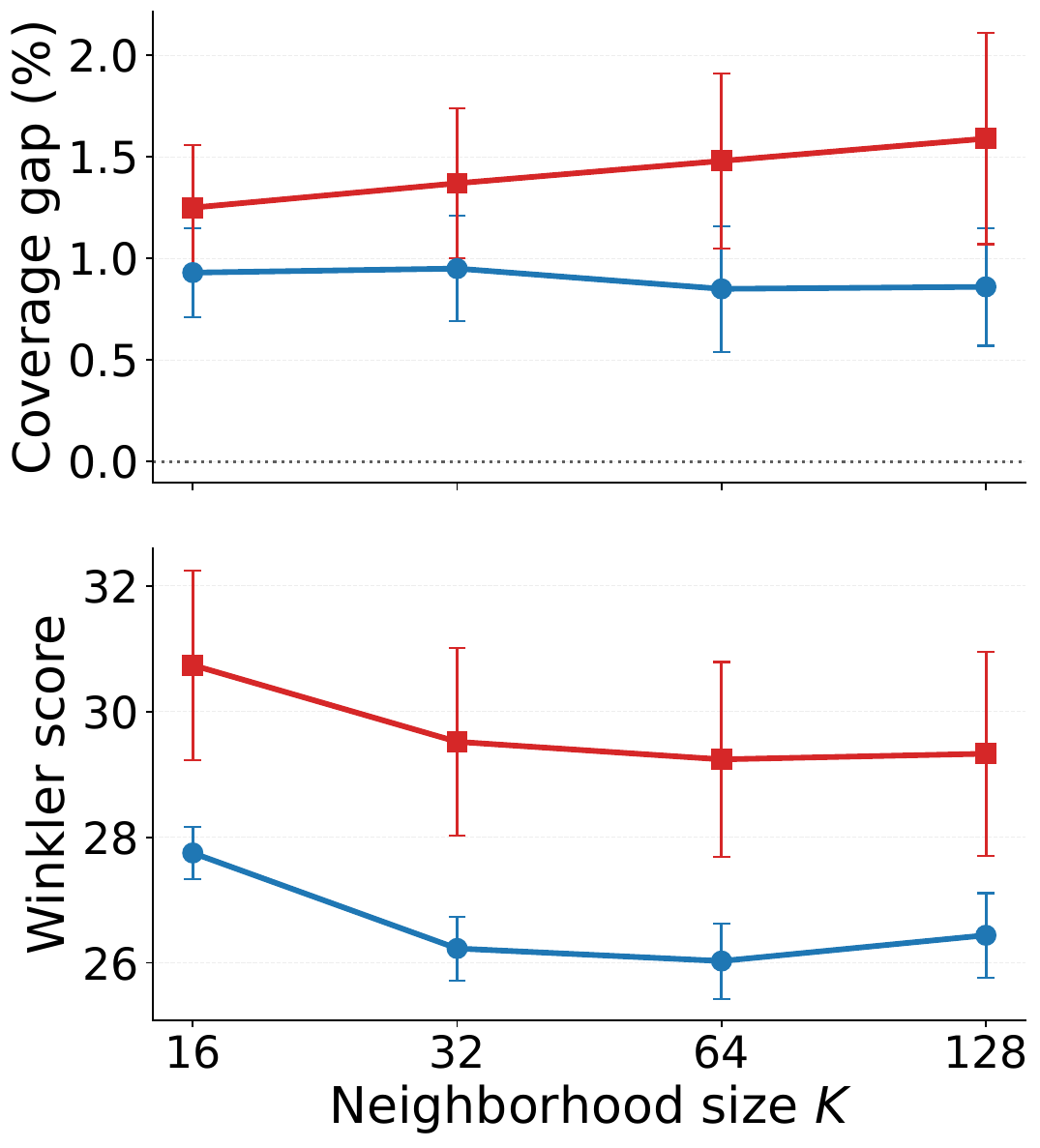}
        \caption{Solar}
        \label{fig:k_sens_solar}
    \end{subfigure}

    \caption{Sensitivity to retrieval size $K$. The figure reports interval width, coverage, and Winkler score across neighborhood sizes.}
    \label{fig:k_sens}
\end{figure}

\paragraph{Retrieval size sensitivity.}
The final ablation varies the neighborhood size $K$ and reports interval width, coverage, and Winkler score in Figure~\ref{fig:k_sens}. Small neighborhoods can make retrieved intervals variable, while very large neighborhoods dilute local information. Across a broad range of $K$, coverage remains stable and Winkler changes smoothly, suggesting that the conformal correction absorbs much of the retrieval variability. The default $K=64$ provides a stable balance between local adaptivity and interval quality.


\section{Conclusion}
In this paper we introduce RCCP, a retrieval--corrected calibration method for time series prediction intervals. Existing time series CP methods improve local calibration by reusing recent, weighted, or localized residuals, but broad weighting or additional adaptation procedures can make the interval less direct and more costly. RCCP addresses this issue by retrieving residuals from similar past prediction contexts and using one-sided residuals to construct an asymmetric local interval. Because the retrieved interval is not a coverage guarantee on its own, RCCP calibrates the remaining normalized retrieval error with a scalar conformal correction. This separates local evidence construction from the final coverage correction. The theoretical analysis characterizes this corrected interval through the normalized retrieval error distribution and gives a coverage gap bound with an asymptotic coverage result. Across benchmark datasets and backbone forecasters, RCCP maintained reliable empirical coverage, improved interval efficiency, reduced severe misses, and kept calibration overhead low. Ablations also showed that retrieval, asymmetry, and correction contribute complementary benefits. These results suggest that retrieval and correction provide a simple and scalable route to locally informative and calibrated uncertainty estimates for time series forecasting.

\section{Limitations}
A key limitation of RCCP is its dependence on the quality of the retrieval representation. Because the retrieval key is derived from the fixed forecaster and its representations, interval efficiency may degrade when the embedding does not capture error-relevant similarity. Future work should also study retrieval--corrected calibration for high-dimensional observations, multivariate targets, and multi-step horizons, where residuals are structured and uncertainty can vary across dimensions and horizons.


\section*{Generative AI Usage Disclosure}

Generative AI tools were used during the preparation of this manuscript for editorial assistance, literature organization, drafting support, and minor code-editing support. These tools were not used to generate datasets, produce experimental results, conduct statistical analyses, or determine scientific claims. All technical content, citations, algorithms, code, experimental results, and conclusions were reviewed and verified by the authors, who take full responsibility for the final manuscript.
\bibliographystyle{ACM-Reference-Format}
\bibliography{ref}


\appendix

\section{Proof}
\label{app:theory}

\setcounter{lemma}{0}
\setcounter{proposition}{0}
\setcounter{corollary}{0}
\setcounter{theorem}{0}

\begin{lemma}[Score-event equivalence]
\label{lem:score_event_equivalance}
For any \(c\ge 0\),
\[
    Y_t\in C_t(c)
    \quad\Longleftrightarrow\quad
    B_t\le c.
\]
\end{lemma}

\begin{proof}
Fix any \(c\ge 0\). By the definition of the asymmetric prediction
interval,
\[
    Y_t\in C_t(c)
    \quad\Longleftrightarrow\quad
    \widehat Y_t-c\tilde R_t^-
    \le Y_t \le
    \widehat Y_t+c\tilde R_t^+ .
\]
Equivalently,
\[
    e_t^+ \le c\tilde R_t^+,
    \qquad
    e_t^- \le c\tilde R_t^- ,
\]
where \(e_t^+=(Y_t-\widehat Y_t)_+\) and
\(e_t^-=(\widehat Y_t-Y_t)_+\). Since \(\tilde R_t^+,\tilde R_t^->0\),

\noindent this is equivalent to
\[
    \max\left\{
        \frac{e_t^+}{\tilde R_t^+},
        \frac{e_t^-}{\tilde R_t^-}
    \right\}
    \le c .
\]

\noindent The left-hand side is precisely the normalized asymmetric score \(B_t\).
Therefore, the above event is equivalent to \(B_t\le c\).
\end{proof}

\vspace{0.5em}
\begin{theorem}[Coverage gap bound of RCCP]
\label{thm:rccp_coverage_gap}
Suppose Assumptions \ref{asm:stability}--\ref{asm:emp_acc} hold. Let
\[
    u_n:=\frac{2(e_n+r_n)}{m}.
\]
If \(u_n<\delta_0\), then
\[
    \left|
        \mathbb P_{\rm test}\{Y_t\in C_t(\widehat c)\}
        -
        q
    \right|
    \le
    \rho_n
    +
    \frac{2L}{m}(e_n+r_n).
\]
\end{theorem}

\begin{proof}
For the upper deviation,
\[
\begin{aligned}
    \widehat F_{\rm cal}(c^*+u_n)
    &\ge F_{\rm cal}(c^*+u_n)-e_n \\
    &\ge q+mu_n-e_n \\
    &= q+e_n+2r_n \\
    &\ge q_n .
\end{aligned}
\]
Hence \(\widehat c\le c^*+u_n\).

\noindent For the lower deviation, take any \(c<c^*-u_n\). Choose
\[
    u\in
    \left(
        u_n,\,
        \min\{c^*-c,\delta_0\}
    \right].
\]
Then \(c\le c^*-u\), and
\[
\begin{aligned}
    \widehat F_{\rm cal}(c)
    &\le F_{\rm cal}(c)+e_n \\
    &\le F_{\rm cal}(c^*-u)+e_n \\
    &\le q-mu+e_n \\
    &< q-mu_n+e_n \\
    &= q-e_n-2r_n \\
    &\le q_n .
\end{aligned}
\]
Thus no \(c<c^*-u_n\) reaches level \(q_n\), so
\[
    \widehat c\ge c^*-u_n.
\]
Combining the two inequalities yields
\[
    |\widehat c-c^*|
    \le
    u_n
    =
    \frac{2(e_n+r_n)}{m}.
\]

\noindent By Lemma~\ref{lem:score_event_equivalance},
\[
\begin{aligned}
 \left|
        \mathbb P_{\rm test}\{Y_t\in C_t(\widehat c)\}
        -
        q
    \right|
    &=
    |F_{\rm test}(\widehat c)-q| \\
    &\le
    |F_{\rm test}(\widehat c)-F_{\rm test}(c^*)|
    +
    |F_{\rm test}(c^*)-q| \\
    &\le
    L|\widehat c-c^*|+\rho_n \\
    &\le
    \rho_n
    +
    \frac{2L}{m}(e_n+r_n).
\end{aligned}
\]
\end{proof}

\vspace{1pt}
\begin{corollary}[Asymptotic coverage of RCCP]
\label{cor:rccp_asymptotic_coverage}
Suppose Assumptions 1--4 hold with fixed constants \(m>0\), \(L<\infty\),
and \(\delta_0>0\). If \(\rho_n\to0\), \(e_n\to0\), and \(r_n\to0\), then
\[
    \left|
        \mathbb P_{\rm test}\{Y_t\in C_t(\widehat c)\}
        -
        q
    \right|\to0.
\]
\end{corollary}

\begin{proof}
Since \(e_n+r_n\to0\), for all sufficiently large \(n\),
\[
    \frac{2(e_n+r_n)}{m}<\delta_0 .
\]
Therefore, Theorem~\ref{thm:rccp_coverage_gap} applies. Hence,
\[
    0\le
    \left|
        \mathbb P_{\rm test}\{Y_t\in C_t(\widehat c)\}
        -
        q
    \right|
    \le
    \rho_n+\frac{2L}{m}(e_n+r_n)
    \to0.
\]
\end{proof}

\section{Experimental Details}

\subsection{Hyperparameter Selection}
\label{app:hyper}
Following prior time series CP benchmarks~\cite{hopcpt,rescp}, model selection is performed on the validation split using the selection seed, and the selected configuration is reused for all 5 evaluation seeds. Table~\ref{tab:method_hyperparameters} summarizes the grids and fixed settings.

\vspace{3pt}
\begin{table}[h]
\centering
\caption{Hyperparameter grids and detailed settings for baselines.}
\label{tab:method_hyperparameters}
\scriptsize
\setlength{\tabcolsep}{3pt}
\renewcommand{\arraystretch}{1.12}
\begin{tabular}{@{}p{0.16\columnwidth}p{0.78\columnwidth}@{}}
\toprule
Method & Grid or fixed setting \tabularnewline
\midrule
SCP & No method-specific hyperparameter \tabularnewline
EnbPI & Past residual window $w\in\{10,50,100,150,200\}$ \tabularnewline
SPCI & Past residual window $w=100$ \tabularnewline
NexCP & Decay parameter $\rho\in\{0.9,0.95,0.99,0.999\}$ \tabularnewline
HopCPT & Learning rate $\eta\in\{10^{-2},10^{-3}\}$, dropout $p\in\{0,0.3\}$, positional encoding $\in\{\mathrm{none},\mathrm{relative}\}$ \tabularnewline
ResCP & $\rho_{\mathrm{ESN}}\in\{0.5,0.75,0.9,0.99,1.1,1.5\}$, $\ell\in\{0.5,0.7,0.8,0.9,0.95,1.0\}$\newline $s_{\mathrm{in}}\in\{0.1,0.25,0.5,0.75,1.0,2.0\}$, $\tau\in\{0.01,0.05,0.1,0.25,0.5,1.0,2.0\}$\newline $w\in\{250,500,1000,2000,4000,8000,\mathrm{all}\}$ \tabularnewline
RCCP & Neighborhood size $K\in\{32,64,128\}$ and retrieval temperature $\tau\in\{0.75,1.0,1.25\}$ \tabularnewline
\bottomrule
\end{tabular}
\end{table}

\vspace{3pt}
For most conformal prediction baselines, the calibration period is split evenly, with one half used for hyperparameter selection and the other half used to construct intervals during selection. ResCP follows its original validation protocol and uses 10\% of the calibration period for validation. For SPCI, fitting a quantile random forest at each time step is computationally limiting at our benchmark scale. We therefore use the implementation adopted in prior benchmarks with a past residual window of 100. HopCPT is trained with AdamW for up to 3,000 epochs and validated every 5 epochs. To keep runtime bounded, early stopping with patience 100 is used. For Air and Solar, ResCP uses the recommended settings from the original implementation.

\definecolor{stdgray}{gray}{0.50}
\definecolor{bestbg}{gray}{0.90}

\newcolumntype{M}{>{\valfont}c}
\newcolumntype{N}{>{\valfont}c}

\begin{table*}[t]
\centering
\caption{
Performance comparison across datasets and significance levels for the LSTM forecaster.
Values are mean $\pm$ standard deviation over random seeds.
For $\Delta$Cov, values closer to zero are better; lower is better for PI-Width, Winkler, and Time.
Shaded Winkler cells mark the lowest value, underlined Winkler values mark the second-lowest value, and shaded $\Delta_{\mathrm{Cov}}$ cells indicate severe undercoverage cases with $\Delta_{\mathrm{Cov}}<-2\%$.
}
\label{tab:alpha_results}

\small
\setlength{\tabcolsep}{0.85pt}
\renewcommand{\arraystretch}{0.98}

\begin{tabular*}{0.97\textwidth}{
@{}
>{\centering\arraybackslash}m{0.055\textwidth}
!{\hspace{3pt}\vrule width 0.35pt\hspace{5pt}}
c
@{\hspace{9pt}}
M
@{\extracolsep{\fill}}
*{7}{N}
@{}}
\toprule
{\headfont $\alpha$}
& {\headfont Dataset}
& {\headfont Metric}
& {\headfont SCP}
& {\headfont EnbPI}
& {\headfont SPCI}
& {\headfont NexCP}
& {\headfont HopCPT}
& {\headfont ResCP}
& {\headfont RCCP (Ours)} \\
\midrule

\multirow[c]{16}{*}{\rotatebox[origin=c]{90}{\large\textbf{$\alpha=0.05$}}}
& \multirow[c]{4}{*}{Air} & $\Delta$Cov & \ms{0.12}{0.02} & \gms{-5.45}{0.08} & \ms{-0.94}{0.50} & \ms{-1.09}{0.03} & \ms{-0.18}{0.08} & \ms{-0.35}{0.11} & \ms{0.18}{0.06} \\
& & PI-Width & \ms{89.99}{0.46} & \ms{70.59}{0.46} & \ms{86.72}{3.12} & \ms{87.79}{0.21} & \ms{89.19}{1.06} & \ms{84.23}{0.98} & \ms{80.57}{0.61} \\
& & Winkler & \ms{152.00}{0.59} & \ms{149.76}{0.35} & \ms{145.75}{0.92} & \ms{134.62}{0.23} & \ms{149.60}{1.31} & \ums{134.25}{0.58} & \bms{117.40}{0.36} \\
& & Time & \ms{36.41}{0.73} & \ms{145.17}{2.56} & \ms{664.27}{1.23} & \ms{35.47}{0.45} & \ms{264.84}{2.22} & \ms{100.78}{1.80} & \ms{71.08}{0.43} \\
\cmidrule(l){2-10}

& \multirow[c]{4}{*}{Solar} & $\Delta$Cov & \ms{0.71}{0.07} & \ms{-1.80}{0.02} & \ms{-0.22}{0.33} & \ms{-0.90}{0.03} & \ms{0.61}{0.91} & \ms{0.54}{0.04} & \ms{0.33}{0.10} \\
& & PI-Width & \ms{57.72}{1.96} & \ms{37.91}{1.29} & \ms{44.11}{2.72} & \ms{41.22}{1.18} & \ms{39.16}{10.21} & \ms{30.65}{1.52} & \ms{25.08}{0.30} \\
& & Winkler & \ms{88.49}{2.28} & \ms{78.73}{2.11} & \ms{77.89}{1.27} & \ms{70.43}{1.82} & \ums{50.26}{9.56} & \ms{53.32}{1.49} & \bms{32.65}{0.54} \\
& & Time & \ms{90.22}{1.29} & \ms{450.55}{3.71} & \ms{1964.12}{22.33} & \ms{98.87}{1.75} & \ms{869.47}{137.85} & \ms{318.53}{5.12} & \ms{167.64}{0.93} \\
\cmidrule(l){2-10}

& \multirow[c]{4}{*}{Wind} & $\Delta$Cov & \ms{0.25}{0.29} & \ms{-1.68}{0.15} & \gms{-2.66}{2.14} & \ms{-0.16}{0.11} & \ms{1.29}{0.30} & \ms{0.75}{0.11} & \ms{0.27}{0.18} \\
& & PI-Width & \ms{86.26}{1.72} & \ms{78.25}{0.25} & \ms{75.75}{10.87} & \ms{83.61}{0.21} & \ms{99.20}{5.51} & \ms{85.39}{0.50} & \ms{76.97}{1.22} \\
& & Winkler & \ms{120.47}{0.42} & \ms{122.76}{0.69} & \ms{131.00}{5.98} & \ms{118.67}{0.51} & \ms{125.34}{3.43} & \ums{111.03}{0.45} & \bms{102.66}{0.41} \\
& & Time & \ms{0.73}{0.03} & \ms{4.66}{0.12} & \ms{19.32}{0.34} & \ms{0.87}{0.03} & \ms{10.16}{1.92} & \ms{4.23}{0.28} & \ms{1.51}{0.01} \\
\cmidrule(l){2-10}

& \multirow[c]{4}{*}{Electricity} & $\Delta$Cov & \ms{3.33}{0.22} & \gms{-2.28}{0.19} & \ms{2.43}{0.96} & \ms{0.58}{0.22} & \ms{2.86}{0.54} & \ms{0.80}{0.27} & \ms{-0.36}{0.79} \\
& & PI-Width & \ms{0.30}{0.02} & \ms{0.18}{0.01} & \ms{0.27}{0.03} & \ms{0.21}{0.01} & \ms{0.28}{0.02} & \ms{0.22}{0.01} & \ms{0.21}{0.01} \\
& & Winkler & \ms{0.32}{0.02} & \ums{0.29}{0.02} & \ms{0.31}{0.02} & \bms{0.28}{0.02} & \ms{0.32}{0.02} & \ums{0.29}{0.02} & \ums{0.29}{0.02} \\
& & Time & \ms{0.21}{0.01} & \ms{1.36}{0.06} & \ms{4.40}{0.07} & \ms{0.25}{0.00} & \ms{4.36}{0.23} & \ms{1.90}{0.43} & \ms{0.23}{0.00} \\

\midrule

\multirow[c]{16}{*}{\rotatebox[origin=c]{90}{\large\textbf{$\alpha=0.15$}}}
& \multirow[c]{4}{*}{Air} & $\Delta$Cov & \ms{0.07}{0.20} & \gms{-6.68}{0.12} & \ms{-0.89}{1.00} & \ms{-0.88}{0.11} & \ms{-0.25}{0.20} & \ms{-0.00}{0.28} & \ms{0.56}{0.21} \\
& & PI-Width & \ms{49.49}{0.39} & \ms{43.83}{0.35} & \ms{49.97}{1.88} & \ms{51.51}{0.28} & \ms{49.46}{0.63} & \ms{49.43}{0.46} & \ms{49.32}{0.45} \\
& & Winkler & \ms{93.42}{0.49} & \ms{91.87}{0.49} & \ms{91.16}{0.25} & \ms{87.60}{0.23} & \ms{92.89}{0.57} & \ums{85.43}{0.38} & \bms{79.49}{0.22} \\
& & Time & \ms{37.97}{1.28} & \ms{146.43}{2.89} & \ms{667.99}{8.91} & \ms{35.69}{0.44} & \ms{265.47}{1.66} & \ms{99.54}{2.34} & \ms{70.89}{0.13} \\
\cmidrule(l){2-10}

& \multirow[c]{4}{*}{Solar} & $\Delta$Cov & \ms{2.02}{0.08} & \ms{-1.66}{0.10} & \ms{1.43}{0.64} & \ms{-0.62}{0.10} & \ms{1.60}{3.21} & \ms{1.53}{0.19} & \ms{1.15}{0.21} \\
& & PI-Width & \ms{23.01}{0.42} & \ms{14.83}{0.66} & \ms{18.31}{1.73} & \ms{18.50}{0.60} & \ms{23.97}{6.13} & \ms{15.06}{0.87} & \ms{15.18}{0.32} \\
& & Winkler & \ms{48.69}{1.11} & \ms{42.87}{1.36} & \ms{42.72}{1.15} & \ms{41.03}{1.12} & \ms{32.97}{5.25} & \ums{31.73}{1.08} & \bms{22.74}{0.59} \\
& & Time & \ms{90.88}{1.44} & \ms{450.11}{5.88} & \ms{1980.33}{15.93} & \ms{98.42}{1.32} & \ms{835.57}{72.30} & \ms{321.50}{3.83} & \ms{168.72}{0.77} \\
\cmidrule(l){2-10}

& \multirow[c]{4}{*}{Wind} & $\Delta$Cov & \ms{0.16}{0.20} & \gms{-2.17}{0.30} & \ms{-0.65}{2.05} & \ms{0.14}{0.17} & \ms{3.38}{2.56} & \ms{2.05}{0.25} & \ms{0.90}{0.23} \\
& & PI-Width & \ms{51.13}{0.73} & \ms{48.62}{0.44} & \ms{48.71}{4.76} & \ms{50.90}{0.51} & \ms{57.87}{6.37} & \ms{50.86}{0.52} & \ms{48.57}{0.74} \\
& & Winkler & \ms{83.04}{0.13} & \ms{83.67}{0.25} & \ms{84.32}{0.65} & \ms{81.50}{0.17} & \ms{83.02}{2.39} & \ums{78.58}{0.27} & \bms{73.03}{0.22} \\
& & Time & \ms{0.72}{0.01} & \ms{4.54}{0.01} & \ms{19.20}{0.04} & \ms{0.86}{0.01} & \ms{9.02}{0.23} & \ms{4.01}{0.23} & \ms{1.51}{0.01} \\
\cmidrule(l){2-10}

& \multirow[c]{4}{*}{Electricity} & $\Delta$Cov & \ms{5.83}{1.60} & \gms{-2.54}{0.70} & \ms{4.13}{1.32} & \ms{1.74}{0.43} & \ms{5.00}{1.09} & \ms{0.69}{0.63} & \ms{0.22}{1.05} \\
& & PI-Width & \ms{0.17}{0.01} & \ms{0.13}{0.01} & \ms{0.15}{0.01} & \ms{0.14}{0.01} & \ms{0.16}{0.01} & \ms{0.14}{0.01} & \ms{0.13}{0.01} \\
& & Winkler & \ms{0.21}{0.01} & \ms{0.21}{0.02} & \ms{0.20}{0.01} & \ms{0.20}{0.01} & \ms{0.21}{0.02} & \ums{0.20}{0.01} & \bms{0.19}{0.01} \\
& & Time & \ms{0.19}{0.00} & \ms{1.21}{0.06} & \ms{4.53}{0.13} & \ms{0.24}{0.01} & \ms{4.00}{0.38} & \ms{1.58}{0.05} & \ms{0.23}{0.01} \\

\bottomrule
\end{tabular*}

\vspace{-1.0mm}
\end{table*}

\section{Additional Results}
\label{app:add_results}

\subsection{Additional Miscoverage Levels}
\label{app:alpha_results}
Table~\ref{tab:alpha_results} reports results at $\alpha=0.05$ and $\alpha=0.15$ using the global LSTM forecaster. This evaluation checks whether the correction remains reliable beyond the main $\alpha=0.1$ setting. RCCP achieves the best Winkler score on the most datasets at both target levels and keeps coverage gaps close to the nominal level, with gaps that are often smaller than those in the main setting.

The results show that RCCP is not tuned to a single miscoverage level. As $\alpha$ changes, baseline methods show larger fluctuations in Winkler score because interval width and coverage violations move together. RCCP remains stable across target levels, maintaining calibrated coverage while producing intervals with competitive width.

\vspace{-5pt}
\subsection{Coverage and Efficiency Tradeoff}
\label{app:pareto_results}
Figure~\ref{fig:pareto_avg} summarizes the main $\alpha=0.1$ results averaged over the four datasets by plotting Winkler score against $|\Delta_{\mathrm{Cov}}|$ and Calibration Time. Methods near the lower-left region achieve stronger joint performance across interval quality, coverage, and computation.

This frontier view highlights the practical value of RCCP, showing that retrieval-corrected calibration resolves the central tradeoff by producing sharper calibrated intervals without the computational burden of heavier adaptive baselines.

\begin{figure}[t]
    \centering
    \includegraphics[
        width=0.72\linewidth,
    ]{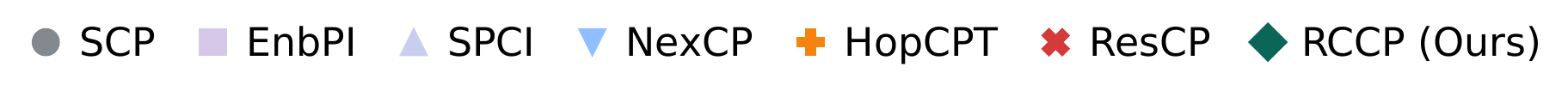}

    \begin{subfigure}[t]{0.485\linewidth}
        \centering
        \includegraphics[width=\linewidth]{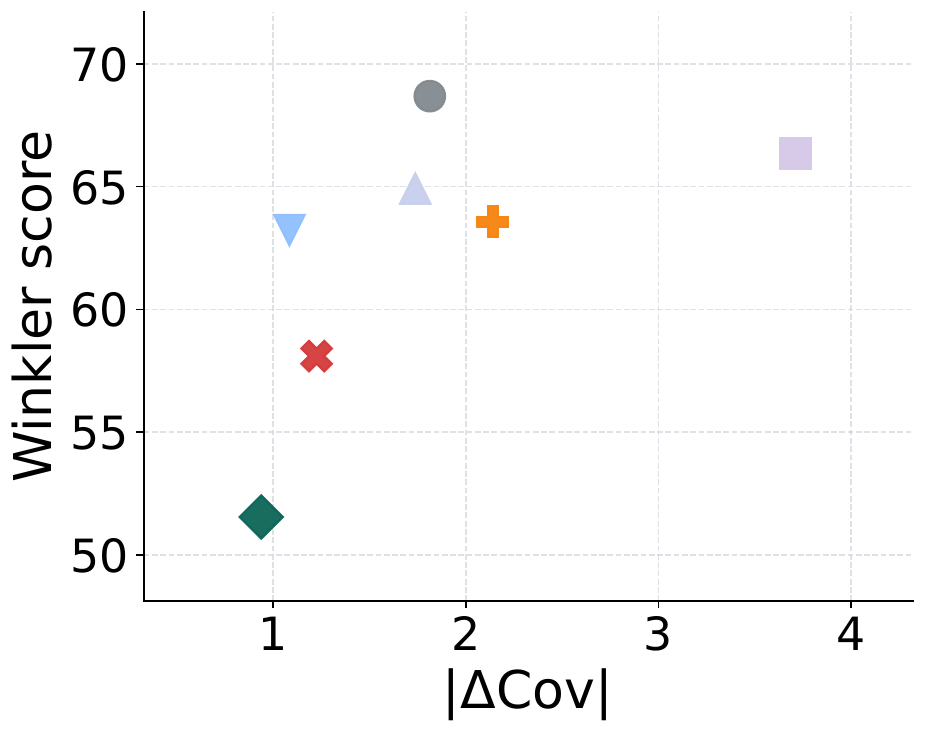}
        \caption{$|\Delta_{\mathrm{Cov}}|$ and Winkler}
        \label{fig:pareto_cov_winkler}
    \end{subfigure}
    \hfill
    \begin{subfigure}[t]{0.485\linewidth}
        \centering
        \includegraphics[width=\linewidth]{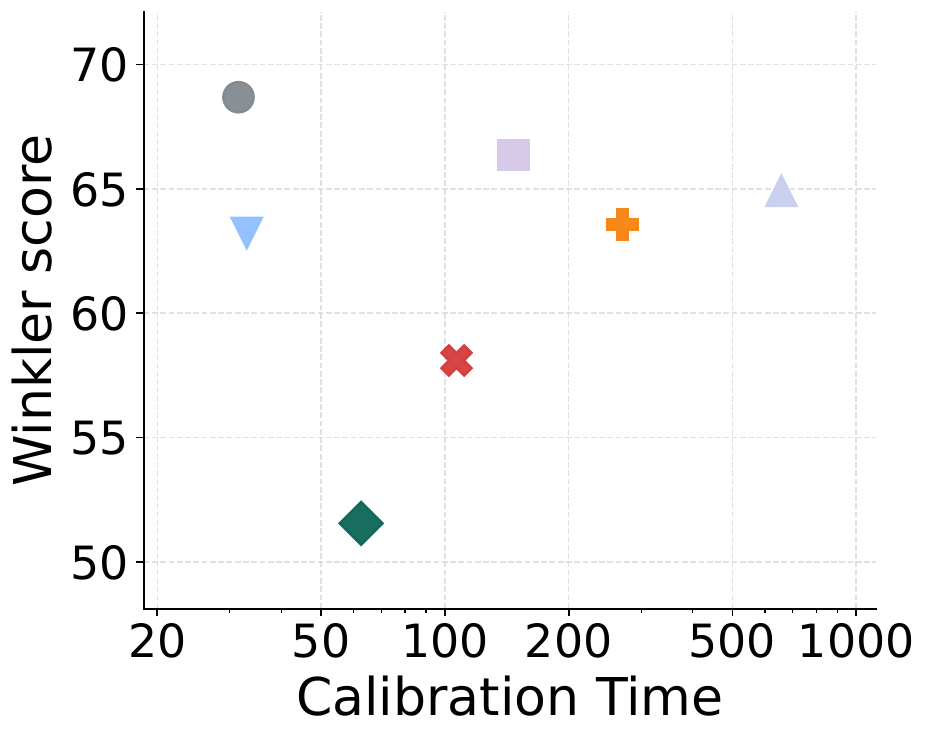}
        \caption{Calibration Time and Winkler}
        \label{fig:pareto_time_winkler}
    \end{subfigure}

    \caption{Average tradeoff across four datasets at $\alpha=0.1$. The plots compare Winkler score with $|\Delta_{\mathrm{Cov}}|$ and Calibration Time.}
    \label{fig:pareto_avg}
\end{figure}

\end{document}